\documentclass[twoside]{article}

\usepackage[accepted]{aistats2022}
\usepackage{amsmath, amsthm, amssymb}
\usepackage{mathtools}
\DeclareMathOperator*{\argmax}{\text{argmax}}
\DeclareMathOperator*{\argmin}{\text{argmin}}
\DeclareMathOperator*{\minimize}{\text{minimize}}

\newcommand{\cA}{\mathcal{A}}

\newcommand{\cF}{\mathcal{F}}
\newcommand{\cG}{\mathcal{G}}

\newcommand{\cC}{\mathcal{C}}

\newcommand{\E}{\mathbb{E}}

\renewcommand{\P}{\mathbb{P}}

\newtheorem{theorem}{Theorem}
\newtheorem*{proposition*}{Proposition}
\newtheorem{proposition}{Proposition}
\newtheorem{assumption}{Assumption}

\newtheorem{lemma}{Lemma}

\newtheorem{corollary}{Corollary}
\newtheorem*{theorem*}{Theorem}

\DeclareMathOperator{\Ber}{Ber}
\DeclareMathOperator{\VC}{VC}
\DeclareMathOperator{\err}{err}

\newtheoremstyle{break}
  {}
  {}
  {\itshape}
  {}
  {\bfseries}
  {:}
  {\newline}
  {}

\theoremstyle{break}

\newtheorem*{assumption*}{Assumptions}
\usepackage[mathscr]{euscript}
\usepackage{subcaption}
\usepackage{algorithm}
\usepackage{algpseudocode}
\usepackage{apptools}
\usepackage[inline,draft,nomargin]{fixme}
\usepackage{enumitem}
\AtAppendix{\counterwithin{lemma}{section}}
\AtAppendix{\counterwithin{assumption}{section}}
\AtAppendix{\counterwithin{theorem}{section}}
\AtAppendix{\counterwithin{corollary}{section}}

\usepackage[round]{natbib}

\begin{document}

\twocolumn[

\aistatstitle{Leveraging Time Irreversibility with Order-Contrastive Pre-training}

\aistatsauthor{ Monica Agrawal* \And Hunter Lang* \And Michael Offin \And Lior Gazit \And David Sontag}

\aistatsaddress{MIT CSAIL \And MIT CSAIL \And  MSKCC \And MSKCC\And MIT CSAIL } ]

\begin{abstract}
  Label-scarce, high-dimensional domains such as healthcare present a challenge for modern machine learning techniques. To overcome the difficulties posed by a lack of labeled data, we explore an ``order-contrastive'' method for self-supervised pre-training on longitudinal data. 
We sample pairs of time segments, switch the order for half of them, and train a model to predict whether a given pair is in the correct order. 
Intuitively, the ordering task allows the model to attend to the \emph{least time-reversible} features (for example, features that indicate progression of a chronic disease).
The same features are often useful for downstream tasks of interest.
To quantify this, we study a simple theoretical setting where we prove a finite-sample guarantee for the downstream error of a representation learned with order-contrastive pre-training.
Empirically, in synthetic and longitudinal healthcare settings, we demonstrate the effectiveness of order-contrastive pre-training in the small-data regime over supervised learning and other self-supervised pre-training baselines.
Our results indicate that pre-training methods \emph{designed for} particular classes of distributions and downstream tasks can improve the performance of self-supervised learning.
\end{abstract}

\section{Introduction}
\label{sec:intro}
The advent of electronic health records has led to an explosion in longitudinal health data. This data can power comparative effectiveness studies, provide clinical decision support, enable retrospective research over real-world outcomes, and inform clinical trial design.
However, longitudinal health data is often complex, unstructured, and high-dimensional and thus untapped.
Typically, limited \emph{labeled} data is available for downstream tasks of interest, and labels can be prohibitively expensive to obtain: long records are tedious to synthesize, comprehension requires domain expertise, and patient privacy regulations limit data-sharing across institutions \citep{agony, xia2012clinical}.
Fortunately, given the large amount of unlabeled data, \emph{self-supervision} is a promising avenue.

In self-supervision, models are first pre-trained to optimize an objective over unlabeled data, with the goal of learning representations that capture important semantic structure about the input data modality.
For example, in masked language modeling for text, the model is trained to predict the identity of randomly masked tokens. 
Performing well at this objective should require a representation of sentence syntax and semantics.
Once pre-trained, self-supervised representations can be used for downstream supervised tasks.
\begin{figure*}
\centering
    \centering
    \includegraphics[width=1.0\textwidth]{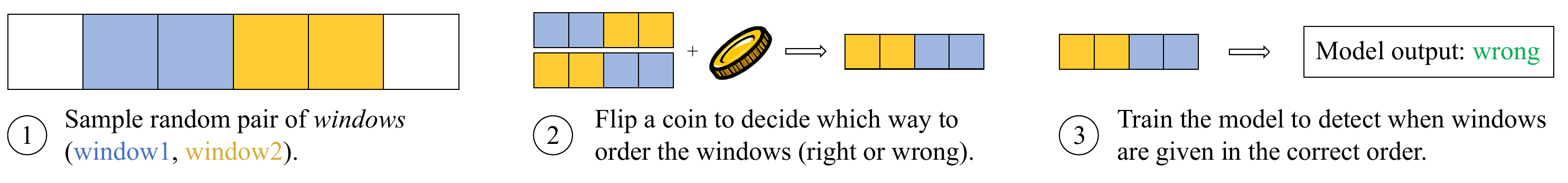}
    \caption{Depiction of the data-generation and learning process for order-contrastive pre-training.
    For each trajectory, a pair of consecutive windows is sampled uniformly at random, flipped with probability 0.5, and presented to the model. The model is trained to predict whether the presented pair of windows is in the correct ($+1$) or incorrect ($-1$) order.}
    \label{fig:ocl-gen-process}
\end{figure*}
However, despite the success of self-supervision across domains, the development of new self-supervised objectives has been a largely heuristic endeavor. 
\emph{Why} does pre-training improve performance on downstream tasks?
Whether this happens depends on both the self-supervised objective and the downstream task itself.
But the assumptions linking the self-supervised objective to the downstream tasks of interest are rarely, if ever, made explicit. 

In this work, we design a self-supervised objective with a \emph{particular class} of data distributions and downstream tasks in mind.
We aim to make explicit the type of distributions and downstream tasks on which we expect this method to work.
We are interested primarily in the types of time-series that arise in longitudinal health data for patients, particularly those with chronic conditions, e.g. cancer, autoimmune disorders, and neurodegenerative diseases. These time-series include long sequences of clinical notes, insurance claims data, biomarker measurements, or combinations thereof. 
A key differentiating feature of these data is that a given trajectory can \emph{change quickly}.
For example, a patient may develop a new symptom between subsequent healthcare visits, or a certain biomarker value (e.g., blood pressure) may dramatically increase. 

Additionally, these changes largely tend to be \emph{irreversible} with respect to time.
For example, once the word ``metastasis'' appears in a clinical note, nearly all subsequent notes tend to comment on the state of that metastasis (so the word ``metastasis'' appears in those notes as well). 

To train a good representation for certain downstream problems (e.g., ``what is the patient's current disease state?''), the self-supervised objective should \emph{attend} to these changes, rather than suppress them.

These properties make such data distributions unsuitable for several existing self-supervised objectives.
For example, \citet{franceschi2019} train a model so that the representation of each time segment is more similar to those of its subsegments than the representation of a randomly chosen segment from another trajectory.
Similar techniques have been used to learn image representations from video: two subsequent video frames are likely to contain the same objects \citep{mobahi2009deep, goroshin2015unsupervised}.
These approaches are all similar to the idea of \emph{slow feature analysis} \citep{wiskott2002slow} for extracting representations of an input signal that change slowly over time.
Representation learning techniques based on the ideas of slow feature analysis are appropriate for some downstream tasks and time-series data types, such as the ones studied in the works above, but not, we argue, for data where the latent variable of interest (such as disease state) can have large changes between subsequent time steps (e.g., between visits to a medical care center). 
Motivated by the example of chronic diseases in healthcare, we focus on the setting where time-irreversible features are highly useful for downstream classification, and where these features may exhibit large changes between subsequent time steps.

In this work, we introduce a self-supervised objective called \emph{order-contrastive pre-training} (OCP).
For each trajectory in the input data, we sample random pairs of time segments, switch the order for half of them, and train a model to predict whether a given pair is in the correct order (\emph{positives}) or in the incorrect order (\emph{negatives}).

This procedure is shown in Figure \ref{fig:ocl-gen-process}. OCP is very similar to an existing technique known as \emph{permutation-contrastive learning}, or PCL \citep{hyvarinen2017nonlinear}.
PCL was also designed to take advantage of temporal dependence between features of the input signal to learn useful representations.
The key difference between these two objectives is in the sampling of the \emph{negatives}.
Where the negatives in OCP are incorrectly-ordered window pairs, the negatives in PCL are \emph{random} window pairs from the same trajectory, and could be in the correct order.
In their simplest forms, the positive samples for the two methods are identical: pairs of consecutive windows in the correct order.

Intuitively, the same time-irreversible features that are useful for the OCP and PCL objectives should also be useful for downstream prediction tasks. 
To formalize and quantify this, we study a class of data distributions motivated by the preceding discussion.
When the representation belongs to a simple hypothesis class (effectively, when the representation is a \emph{feature selector}), we prove a finite-sample bound on the \emph{downstream} error of a representation learned using OCP.
Although this setting is much simpler than those that appear in similar work (it involves \emph{linear}, rather than nonlinear, representations of the input data), we show that this model still admits interesting behavior.
In particular, we give an example of a data distribution in this setup where OCP and PCL provably learn different representations.
Additionally, this model indicates that even when two methods have the same performance with infinite unlabeled data, there is an unlabeled-sample-complexity benefit to using a ``clean'' distribution of negatives, which matches well with prior work on other contrastive learning algorithms \citep{chuang2020debiased}.

We supplement this motivating theoretical study with experiments on real-world time-series data.
Our results indicate that for the types of data and tasks discussed above, both OCP and PCL representations can enjoy better downstream prediction performance than those trained using existing self-supervised baselines.
Moreover, complementing our theoretical results, we show a real-world scenario where OCP outperforms PCL in the low labeled-data regime despite the seemingly minor difference between the two objectives.
Given that OCP and PCL only differ slightly in their negative sampling, these results give further theoretical and empirical evidence for the importance of the negative sampling details in contrastive learning, complementing several recent works \citep{chuang2020debiased, robinson2021contrastive, liu2021contrastive}.

\section{Order-pretraining algorithm}
We suppose each data point $X$ is a time series, $X = (X^1, \ldots, X^{\tau})$,
where $\tau$ is the number of sample points and may vary with $X$.
We also suppose the samples take values in some common set $\mathcal{X}$.
Let a window $w$ be an element of $\{1,\dots,\tau\}$, and let $X^w$ be the corresponding element of $X$.\footnote{For simplicity, we only consider windows of size 1. Our results straightforwardly generalize to windows of arbitrary size $\ell$, where $w=(w_1,\dots, w_\ell)$ is a subinterval of $\{1,\dots,\tau\}$ and $X^w=(X^{w_1},\dots, X^{w_\ell})$.}

Given a trajectory $X$, we use the following generative process to sample a data point $(Z,Y)$ for our contrastive task.
First, $Y$ is chosen uniformly at random from $\{-1,1\}$.
Next, random windows $W$ and $W'$ are chosen (in a manner explained below).
The segments $X^W$ and $X^{W'}$ corresponding to windows $W,W'$ are combined into a tuple $Z$.
The pair $(Z,Y)$ is then a sample for the contrastive task.
A model $h$, given by a composition of a classifier $c\in \cC$ and a representation $g\in\cG$, is trained to predict $Y$ from $Z$:
\begin{equation}
  \label{eqn:pretraining-obj}
  \minimize_{h = (c,g)} R_{ord}(h) := \E_{(Z,Y)}[h(Z) \ne Y]
\end{equation}
Here $h(Z) = h(X^{W}, X^{W'}) = c(g(X^W), g(X^{W'}))$.
That is, $h$ first computes the representation $g$ for each window, then uses a classifier $c$ to predict whether the tuple $Z$ is in the correct order.
The representation $g$ can then be re-used on a downstream task.
The remaining design choice is to specify the process for sampling windows.

\paragraph{Order-contrastive pre-training.}
A simple choice for sampling random windows $W$, $W'$ is to sample a random pair $(W,W+1)$ in the \emph{correct} order when $Y=1$, and $(W+1,W)$ in the \emph{incorrect} order when $Y = -1$.
We refer to the optimization problem \eqref{eqn:pretraining-obj} with this choice of sampling as \emph{order-contrastive pre-training} (OCP).
This can easily be generalized to non-consecutive window pairs.
The pretraining task \eqref{eqn:pretraining-obj} is thus to \emph{contrast} windows in the correct order with windows in the incorrect order.

\paragraph{Permutation-contrastive learning.}
Another simple choice is to again sample a pair $(W,W+1)$ in the correct order when $Y=1$, but sample a \emph{random pair} when $Y=-1$.
This is the data generation process for \emph{permutation-contrastive learning} \citep{hyvarinen2017nonlinear}.
Note that the only differences between OCP and PCL are that in PCL, (i) the negative samples ($Y=-1)$ need not be consecutive, and (ii) some are in the correct order. 
The distributions of positive samples are identical.
We refer to the procedure \eqref{eqn:pretraining-obj} with this sampling as \emph{permutation-contrastive learning} (PCL).
This exactly matches the contrastive sample distribution in \citet[][equations (10)-(11)]{hyvarinen2017nonlinear}.
Here, the pretraining task is to contrast consecutive windows in the correct order versus random window pairs.

\paragraph{Comparison.} These two sampling methods seem very similar---they only differ slightly in the distribution of negatives (i.e., conditioned on $Y=-1$).
However, we show theoretically and empirically in the following sections that they can learn very different representations when used in \eqref{eqn:pretraining-obj}, and they can have different unlabeled sample complexities even if they eventually find the same representation.
This gives further evidence of the importance of negative sampling for contrastive learning methods (see, e.g., \citet{chuang2020debiased}).
We give a finite-sample bound for the downstream classification performance of a representation learned using OCP in a simple setup motivated by time series data and predictive tasks in healthcare.

\section{Finite-sample guarantee for time-irreversible features}
\label{sec:setup}

In this section, we study a class of distributions motivated by applications to time-series data in healthcare.
We assume for simplicity that each $X^t \in \mathcal{X} = \{0,1\}^d$. 
We identify a set of four assumptions for which we can prove a finite sample guarantee for the set of \emph{feature selector} representations $\cG$. Here we use $\cF$ to refer to the downstream hypothesis class, and we overload $Y$ to refer to the \emph{downstream} label of interest.

\begin{assumption}
\label{asmp:s-stays-active}
There exists a set $S$ of time-irreversible features. Formally,  $\forall i\in S, \forall t$, $\P[X_i^t = 1, X_i^{t+1} = 0] = 0$.
\end{assumption}

\begin{assumption}
\label{asmp:time-reversible}
When the features in $S$ are not changing, the other features are time-reversible. More formally, for all $t$, and all $v, v' \in \{0,1\}^d$, if $v_S = v'_S$, $\P[X^t = v, X^{t+1} = v'] = \P[X^t = v', X^{t+1} = v]$.
\end{assumption}
\begin{assumption}
  \label{asmp:U-asmp}
  There are no ``redundant'' features in $S$. For all $U$ such that $S\not\subset U$, there exists $t$ and $v\in \{0,1\}^d$, $v'\in \{0,1\}^d$, with $v_{U\cap S} = v'_{U\cap S}$ and:
  \begin{align*}
  \min&\P[X^t_S \subsetneq X^{t+1}_S, X^t_U = v_U, X^{t+1}_U = v'_U];\\
     &\P[X^t_S \subsetneq X^{t+1}_S, X^t_U = v'_U, X^{t+1}_U = v_U]) > 0.
  \end{align*}
\end{assumption}
\begin{assumption}
  \label{asmp:S-suitable}
The features $S$ are suitable for the downstream classification task (here $Y$ refers to the downstream label):
    \[
    \argmin_{f \in \cF} \P[f(X_S) \ne Y] = \argmin_{f \in \cF, g\in \cG} \P[(f\circ g)(X) \ne Y]
    \]
\end{assumption}
\vspace{-1em}
Intuitively, the first two assumptions guarantee that the feature-set $S$ is an optimal choice of representation for the OCP pretraining objective (when $\cG$ is the class of feature selector representations), and the third (more technical) assumption guarantees that the optimum is unique (e.g., by preventing the possibility that pretraining leaves out a feature in $S$ that is redundant for the ordering objective, but useful downstream).
The last assumption ensures that the features $S$ are suitable for the downstream classification task on the population: the loss achievable by the best $f\in \cF$ using $X_S$ as the representation is the same as the loss achievable by using the best $(f,g)$ pair.

When these assumptions are satisfied, we prove a finite-sample bound for a model pretrained using OCP.
The bound only depends on the VC-dimension of the \emph{downstream} hypothesis class, $\VC(\cF)$, rather than on $\VC(\cF \times \cG)$.

Like some results in the nonlinear ICA literature (e.g., \citet{hyvarinen2017nonlinear}) our results only apply in the regime where there is enough unlabeled data to identify the ``correct'' representation.
It's then immediate that only $\cF$ factors in to the labeled-data dependence.
However, our model also allows us to give upper bounds on the \emph{amount} of unlabeled data required to reach that regime.
This allows us to more rigorously study other aspects of contrastive pretraining, such the role of \emph{bias} in the negative distribution, which has been shown to affect the performance of other contrastive learning algorithms \citep{chuang2020debiased}.

We now give a simple example of a class of distributions satisfying these assumptions, grounded in our running application of health time-series data.
Despite its simplicity, our findings suggest that this model allows for several interesting phenomena that also occur in practice, which could make it useful for further study of contrastive learning methods on time-series.

\subsection{Extraction example}
\label{sec:model}
A common task in clinical informatics is to extract for each time $t$ the patient's structured disease stage, which enables downstream clinical research \citep{kehl2019assessment, kehl2020natural}.
Each time point $X^t$ could be an encoding of the clinical note from a patient's visit at time $t$. 
Let $Y^t \in \{0,1\}$ be the observed label for time point $X^t$.
The end goal is to train a model $f$ over a representation $g$ to minimize the downstream risk:
\begin{equation*}
    \minimize_{(f,g)}R(f,g) := \E_{(X^T,Y^T)}[f(g(X^T)) \ne Y^T].
\end{equation*}
Here we make a prediction for every time point, and the expectation is over the time index $T$ as well as the trajectory $(X,Y)$.
\paragraph{Model.} For each $A \subset [d]$, we denote by $X^t_A$ the random variable corresponding to indices $A$ at time $t$. Suppose the set of feature indices $[d]$ is partitioned into three types of features:
\begin{itemize}{\itemsep=1em}
    \item A set $S \subset [d]$ of time-irreversible features.
    We also assume that each $i\in S$ has a nonzero probability of activating on its own, without the other features in $S$.
    That is, for each $i\in S$ there exists $t$ with $\P[X^t_i = 0, X^{t+1}_i = 1, X^t_{S \setminus \{i\}} = X^{t+1}_{S \setminus \{i\}}] > 0$.
    This ensures that assumption \ref{asmp:U-asmp} is satisfied.
    Such features include the onset/progression of chronic conditions and markers of aging \citep{pierson2019inferring}.
    For example, appearance of the word ``metastasis'' in a clinical note.
    \item Noisy versions $\hat{S}$ of $S$: for each $j\in \hat{S}$, there exists $i \in S$ with $\P[X^t_j = X^t_i] = (1-\epsilon_i)$, with $\epsilon_i > 0$, for all $t$. 
    Additionally, $X^t_j$ is conditionally independent of the other variables (for all times) given its parent variable $X^t_i$. 
    For example, the presence of certain interactions with the health system---such as deciding to attend physical therapy---may be a noisy reflection of the patient's true disease state, which is captured by $X^t_S$. 
    \item Background, reversible features $B$: features such that for all $t$ and all $v,v' \in \{0,1\}^d$, 
    \begin{align*}
    &\P\left[(X^t_B, X^t_{[d]\setminus B}) = (v_B, v_{[d]\setminus B}), \right.\\
    &\ \ \ \left.(X^{t+1}_B, X^{t+1}_{[d]\setminus B}) = (v'_B, v'_{[d]\setminus B})\right] = \\
    &\P\left[(X^t_B, X^t_{[d]\setminus B}) = (v'_B, v_{[d]\setminus B}), \right.\\
    &\ \ \ \left.(X^{t+1}_B, X^{t+1}_{[d]\setminus B}) = (v_B, v'_{[d]\setminus B})\right]
    \end{align*}
    Consider, for example, common words such as ``and'', ``chart'', etc., in a clinical note, whose presence or absence gives no order information.
\end{itemize}
Note that we \emph{do not} make any independence assumptions between the features in this example other than the ones mentioned above.

We prove in Theorem \ref{thm:asmp-example} that this example satisfies Assumptions \ref{asmp:s-stays-active}-\ref{asmp:U-asmp}.
Assumption \ref{asmp:S-suitable} is true by design when e.g. $Y^t = f(X^t_S)$ (or a noisy version thereof).
We now provide a simple finite-sample bound for OCP when Assumptions \ref{asmp:s-stays-active}-\ref{asmp:S-suitable} are satisfied and the representation class $\cG$ is the set of \emph{feature selectors}.
 
\subsection{Finite sample bound} 
Suppose we observe a large set of $m$ \emph{unlabeled} data points $\{X_i\}_{i=1}^m$ drawn independently from the marginal distribution of $X$, and a much smaller set of $n$ \emph{labeled} data $\{(X_i, Y_i)\}_{i=1}^n$ drawn independently from the joint distribution of $X$ and the downstream label $Y$ (now we use $Y$ to denote the downstream label rather than the pre-training label).

Let the representation hypothesis class $\cG = \{U \subset [d] : |U| = |S| = d_0\}$.
That is, our representation will \emph{select features} $U \subset [d]$ to be used downstream.
We identify each set of indices $g\in \cG$ with the mapping $\{0,1\}^d\to \{0,1\}^{d_0}$ given by projection onto those indices.
Let $\cF$ be the downstream hypothesis class, with $f: \{0,1\}^{d_0} \to \{0,1\}$.\footnote{Assume $\cF$ is closed under permutations of the input dimensions.
This ensures that we only need to identify the features belonging to $S$, and don't need to put them in a particular order to do well at downstream prediction. The set of linear hypotheses has this property.}
The mapping $f\circ g: \{0,1\}^d\to \{0,1\}$ first selects the input features represented by $g$, then passes the values of these features through $f$.

Our end goal is to design a learning algorithm $\cA$ with a \emph{downstream excess risk bound}. If $(\hat{f}, \hat{g})$ is the hypothesis output by $\cA$, we want an upper bound on the excess risk:
$R(\hat{f}, \hat{g}) - \min_{f\in \cF, g\in \cG} R(f,g).$
If we let $\cA_{ds}$ be empirical risk minimization (ERM) over $\cF\times\cG$ on the small labeled sample (i.e., the method directly optimizing the downstream objective over $\cF \times \cG$ without pre-training),\footnote{Assume for simplicity that for each trajectory $(X,Y)$, a single time $T$ is chosen uniformly at random and $\{(X^T, Y^T)\}$ are passed to the learner, so the learner sees i.i.d. samples. A more detailed treatment would handle the dependence between multiple time points to get bounds that decrease as $\tilde{O}(1/\sqrt{n\tau})$ when possible \citep[e.g.,][]{mohri2010stability}.} a standard result (e.g., \citet[][Theorem 6.8]{shalev2014understanding}) implies: 
\begin{equation}
    \label{eqn:direct-downstream-bound}
    R(\hat{f}, \hat{g}) - \min_{f\in \cF, g\in \cG} R(f,g) \le O\left(\sqrt{\frac{\VC(\cF \times \cG))}{n}}\right)
\end{equation}
with high probability over the sampling of the data.
On the other hand, let $\cA_{pt}$ be the algorithm that first uses unlabeled data to \emph{pre-train} a representation $\hat{g}$ by minimizing \eqref{eqn:pretraining-obj}, then minimizes the downstream risk over $\cF \times \{\hat{g}\}$ (i.e., a 2-phase ERM learner).
The following theorem states that under Assumptions \ref{asmp:s-stays-active}-\ref{asmp:S-suitable}, we can give a more parsimonious upper bound on the excess risk.
In what follows, we use $R_{ord}(g)$ to refer to $\inf_{c\in \cC} R_{ord}(c,g)$, and likewise for $\hat{R}_{ord}$.
We give details on the choice of $\cC$ in Appendix \ref{apdx:theory}.
Since $\cG$ is the class of feature selectors and we assumed $\cF$ is closed under permutations of the input dimensions, we sometimes replace $g\in \cG$ below with sets $U \subset [d]$.
\begin{theorem}
\label{thm:finite-sample}
Suppose Assumptions \ref{asmp:s-stays-active}-\ref{asmp:S-suitable} are satisfied, and let $\epsilon_0 = \min_{U : S\not \subset U} R_{ord}(U) - R_{ord}(S)$ be the difference in OCP error between $S$ and the next-best representation.
Suppose we have an unlabeled dataset of $m$ i.i.d. pretraining points $\{(Z_i, Y_i)\}_{i=1}^m$, with:
\[
m > \frac{2\left(\log {d \choose |S|} + \log \frac{4}{\delta}\right)}{\epsilon_0^2},
\]
and a labeled dataset of $n$ downstream training points $\{(X^t_i, Y^t_i)\}_{i=1}^n$.
Let $\cG$ be all sets of size $|S|$ features chosen from the full set of $d$ features.
Let $\hat{g}$ be the minimizer of the empirical OCP pretraining objective:
\[
\hat{g} = \argmin_{g \in \cG} \hat{R}_{ord}(g)
\]
Let $\hat{f}$ be the minimizer of the empirical downstream objective when using the fixed representation $\hat{g}$:
\[
\hat{f} = \argmin_{f \in \cF} \hat{R}(f, \hat{g})
\]
Then for any $\delta > 0$, with probability at least $1-\delta$, $(\hat{f}, \hat{g})$ has excess risk:
\begin{equation}
\label{eqn:pretrain-bound}
\small
R(\hat{f}, \hat{g}) - \inf_{(f,g) \in \cF \times \cG}R(f,g) \le O\left(\sqrt{\frac{\VC(\cF) + \log\frac{1}{\delta}}{n}}\right)
\end{equation}
\end{theorem}
\begin{proof}[Proof (sketch)]
Assumptions \ref{asmp:s-stays-active}-\ref{asmp:time-reversible} are used to show that $S$ is one of the optima for the population OCP objective, i.e., that
\[
S \in \argmin_{g\in\cG} R_{ord}(g).
\]
Then, Assumption \ref{asmp:U-asmp} is used to show that $S$ is actually the \emph{unique} optimum of size $|S|$ (and hence $\epsilon_0 > 0$). 
The condition on $m$ (the amount of pretraining data), together with a standard learning bound for finite classes, is enough to guarantee that OCP identifies $S$ with high probability over the sampling of the pretraining data.
That is, the choice of $m$ guarantees that with high probability,
\[
S = \argmin_{g\in\cG} \hat{R}_{ord}(g).
\]
Assumption \ref{asmp:S-suitable} guarantees that choosing $S$ in the pretraining step does not incur additional error on the downstream task compared to the optimal $(f,g)$ pair (since it states that $S$ is the optimal $g$ for the \emph{downstream} task).
The result follows from a standard uniform convergence bound (e.g., \citet[Theorem 6.8]{shalev2014understanding}) applied to $\cF$.
The full proof is given in Appendix \ref{apdx:theory}.
\end{proof}

The pretrain + finetune bound \eqref{eqn:pretrain-bound} has a better dependence on the labeled dataset size $n$ than the downstream ERM learner bound \eqref{eqn:direct-downstream-bound}.
The large unlabeled dataset allows for the learning of a good representation $\hat{g}$ without using any labeled data.
Even in this simple feature selection setting, this bound may be much tighter than the direct-downstream bound when $\cF$ is a fairly complex hypothesis class and $d_0 \ll d$. Even for $\cF$ linear, $\cF \times \cG$ is roughly $\Theta(d_0\log(d/d_0))$ (see e.g. \citep{abramovich2018high}), so \eqref{eqn:pretrain-bound} can even save over \eqref{eqn:direct-downstream-bound} in this case. 
In fact, we show in Section \ref{sec:experiments} that OCP can improve the performance of sparse linear models in a real-world low-labeled data setting (compared to direct downstream prediction without pre-training).

In this section we gave a simple example of a class of distributions, together with an \emph{assumption} linking the distribution to the downstream task (Assumption \ref{asmp:S-suitable}---the time-irreversible features are the most useful ones for downstream classification) for which we can \emph{prove} that OCP gives a more parsimonious bound on the labeled sample complexity.
While the example in Section \ref{sec:model} seems straightforward, we show now that it still admits interesting behavior.
In particular, there are distributions that satisfy Assumptions \ref{asmp:s-stays-active}-\ref{asmp:U-asmp}, but where PCL and OCP learn different representations.

\paragraph{PCL versus OCP: different infinite-data optima.}
\label{sec:pcl-diff}
There are examples of the model from Section \ref{sec:model} where PCL and OCP learn provably different representations even with infinite unlabeled samples, despite the minor difference in their sampling schemes.\footnote{The example we use includes \emph{nonstationary} features, which violates the assumptions under which PCL is proven in \citet{hyvarinen2017nonlinear} to find the ``right'' representation, so this does not contradict those results.}
Intuitively, the existence of a periodic feature (such as a procedure always performed at a particular time of day) is strongly predictive of whether two samples are consecutive, but need not be predictive of whether a pair of consecutive samples are in the correct order.
Concretely, consider a feature $X_i$ such that $X_i^{t+1} = 1-X_i^t$, and $X_i^1 \sim \Ber(0.5)$. Inclusion of this feature doesn't violate Assumptions \ref{asmp:s-stays-active}-\ref{asmp:U-asmp}---indeed, $X_i$ would qualify as a ``background'' feature under our model---so Theorem \ref{thm:finite-sample} guarantees that OCP finds the correct representation.
However, in PCL, every non-consecutive sample is a negative.
But only non-consecutive samples can have $X_i^t = X_i^{t'}$, so $X_i^t$ is helpful for the PCL objective.
We treat this example more formally in Appendix \ref{apdx:theory}, but our synthetic results in Section \ref{sec:experiments} also show that a background periodic feature can affect the PCL representation.

\paragraph{``Debiased'' negatives.}
PCL has some negatives that are actually in the correct order.
Prior work on contrastive learning has called this ``bias'' in the negative distribution \citep{chuang2020debiased}.
What's the role of this ``bias?'' 
Does it affect the learned representations? 
Does it affect the amount of \emph{unlabeled} data required to find a good representation?
For distributions satisfying assumptions \ref{asmp:s-stays-active}-\ref{asmp:U-asmp} and when $\cG$ is the class of feature-selectors, we answer these questions in the negative and positive, respectively.

In particular, consider the analogue of OCP that instead of \emph{always} choosing $(W+1, W)$ when $Y=-1$ ($Y$ as used in  OCP, not the downstream label), instead just chooses a random pair $(W, W')$ with $|W-W'| = 1$ (i.e., a random consecutive pair).
We refer to this as OCP-biased, since some of the negatives are actually in the correct order.
However, the following theorem shows the estimator obtained by minimizing this objective is \emph{not} biased in a statistical sense:
\begin{theorem}[informal]
When assumptions \ref{asmp:s-stays-active}-\ref{asmp:U-asmp} are satisfied, $S$ is also the unique optimal representation for OCP-biased.
\end{theorem}

However, it \emph{does} affect the bound on unlabeled sample complexity required to obtain a good representation:
\begin{proposition}[informal]
The upper bound on the sample complexity required for OCP-biased to identify $S$ is worse than the upper bound for OCP.
\end{proposition}
While this proposition only compares upper bounds, our synthetic experiments (see Figure \ref{fig:synthetic}) indicate that OCP is more sample-efficient than OCP-biased.
We prove these results in Appendix \ref{apdx:theory}.
\section{Experiments}
\label{sec:experiments}

\subsection{Synthetic data}
\label{subsec:synth-data}

\begin{figure}[t]
\centering
\begin{subfigure}{\linewidth}
  \centering
    \includegraphics[width=.7\linewidth]{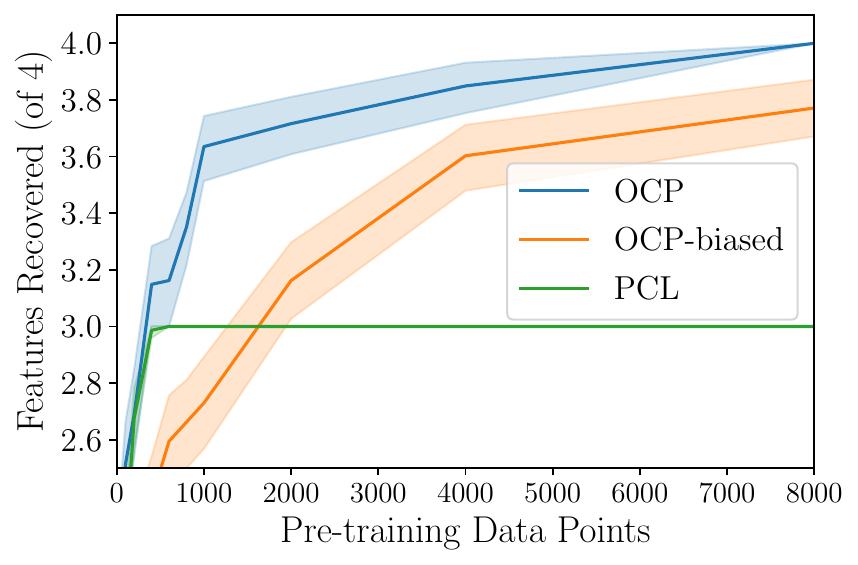}
  \label{fig:synthetica}
\end{subfigure}\\
\begin{subfigure}{\linewidth}
  \centering
    \includegraphics[width=.7\linewidth]{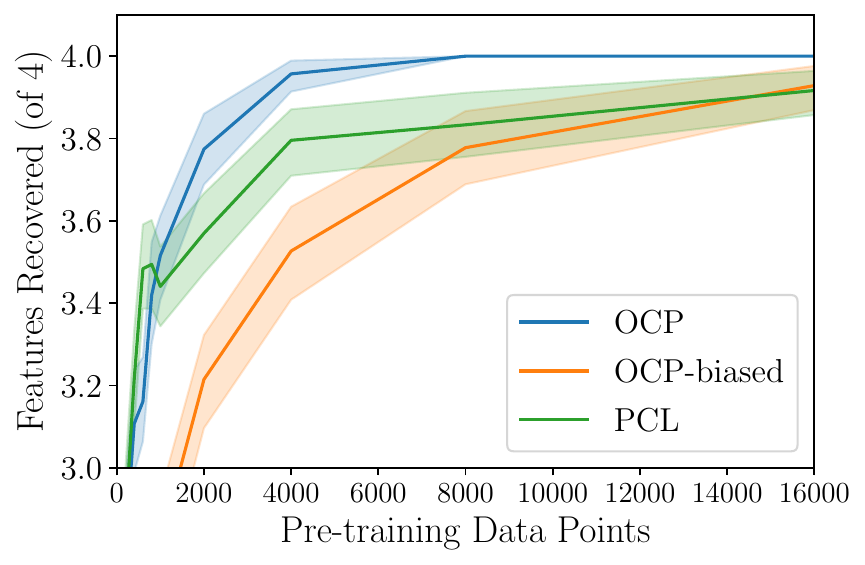}
  \label{fig:syntheticb}
\end{subfigure}
\setlength{\belowcaptionskip}{-10pt}
\caption{Synthetic experiments admitting different behavior across pre-training setups. (Top): PCL is unable to ever recover all four true features in $S$; (Bottom): PCL recovers the true representation, but requires higher sample complexity than OCP.}
\label{fig:synthetic}
\end{figure}

\label{sec:synthetic}
We demonstrate the importance of negative sampling over two synthetic datasets from the model in Section \ref{sec:setup}.
Each distribution contains $|S|=4$ alongside a number of noisy features. 
We generate pre-training datasets of different sizes (50 to 16,000) and sample pairs from each dataset according to OCP, PCL, and OCP-biased. 
We then conduct a logistic regression with L0 penalty over the sample pairs and analyze how many variables in $S$ were correctly recovered. 
The top panel of Figure \ref{fig:synthetic} shows a distribution where PCL does not recover $S$ in the infinite data limit---this distribution includes a periodic background feature in $B$ that is selected by PCL.
The bottom panel of Figure \ref{fig:synthetic} shows a distribution where PCL is able to recover all of $S$, but requires a larger sample complexity than OCP.
In both cases, OCP and OCP-biased find the same representation, but the former has better dependence on unlabeled data.
We provide the details and explanations for these experiments in Appendix \ref{apdx:synth-experiments}.


\subsection{Real-world data}
\label{subsec:real-data}
We show OCP yields significant improvements in the low-label regime on extraction from clinical notes. 
\vspace{-1em}
\paragraph{Progression dataset.}
\label{sec:dataset}
We utilize a dataset of fully de-identified clinical notes from Memorial Sloan Kettering Cancer Center. This research was reviewed by the MIT Committee on the Use of Humans as Experimental Subjects and determined to be IRB-exempt. The dataset contains data for 82,839 patients with cancer, with a median of 12 radiology notes each. Each radiology note focuses on one body area (e.g., chest CT scan). In addition, we have a subset of 135 patients with progressive lung cancer with 1095 labeled radiology notes. Each note was labeled post-hoc by a dedicated thoracic oncologist as `indicating progression' (19\%), `not indicating progression' (79.5\%),  or `ambiguous' (1.5\%).

\begin{figure*}
\centering
\begin{subfigure}{.65\textwidth}
  \centering
    \caption{Mean note-level AUC of regularized logistic regression over different dataset sizes. Averaged over the 5 folds, performance was optimal for each dataset size when restricted to the features with nonzero coefficients recovered by OCP.}
\begin{tabular}{l|lllll|}
\cline{2-6}  & \multicolumn{5}{c|}{\textit{Fraction of training data}}     \\ \hline
\multicolumn{1}{|l|}{\textit{Available features}} & \multicolumn{1}{l}{\textbf{1}} & \multicolumn{1}{l}{\textbf{1/2}} & \multicolumn{1}{l}{\textbf{1/4}} & \multicolumn{1}{l}{\textbf{1/8}} & \multicolumn{1}{l|}{\textbf{1/16}} \\ \hline
    \multicolumn{1}{|l|}{\textbf{OCP subset}}     & 0.864                           & 0.860                             & 0.847                             & 0.808                             & 0.786                             \\ \hline 
    \multicolumn{1}{|l|}{\textbf{All features}}   & 0.856                           & 0.851                             & 0.818                             & 0.723                             & 0.726                             \\\hline 

\multicolumn{1}{|l|}{\textbf{Most common}}  & 0.767                           & 0.767                             & 0.728                             & 0.687                             & 0.658                            \\
\hline
\multicolumn{1}{|l|}{\textbf{Random subset}}  & 0.740                           & 0.747                             & 0.727                             & 0.639                             & 0.634                            \\
\hline
\end{tabular}

  \label{fig:sub1}
\end{subfigure}%
\hfill
\begin{subfigure}{.30\textwidth}
  \centering
    \caption{Example features that OCP selected (top) or excluded (bottom) for downstream prediction.}
    \includegraphics[width=120pt]{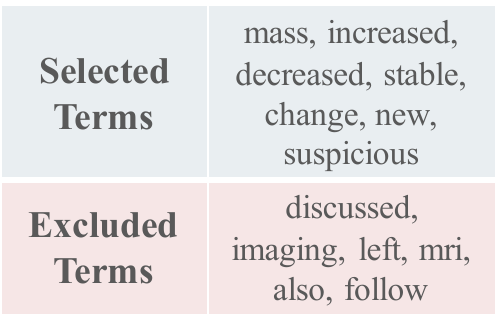}

  \label{fig:sub2}
\end{subfigure}
\setlength{\belowcaptionskip}{-10pt}
\caption{Linear representation space experiment to validate assumptions of our model apply to real-world data. Quantitatively, we find downstream wins from restricting the model feature space to those found useful for the order-contrastive task. Qualitatively, the features important for the order pre-training are the same we would expect to be useful for the downstream extraction task.}
\label{fig:test}
\end{figure*}

\paragraph{Experimental setup.} We investigate extraction of these binary progression labels from the \textit{Impression} section of the  note. The labeled data was split via 5-fold cross-validation: each fold contained sets of sizes 64\% (train), 16\% (validation), and 20\% (test); for a given fold, no patient examples were ever split between sets.
On each fold, we used the test set to benchmark models trained using different amounts of the labeled training data: from just 5 training patients ($\frac{1}{16}$) to all of the training patients. 
We excluded patients with downstream labels from pretraining. 
For contrastive pre-training schemes, a pretraining window pair was sampled once per each unique body area (e.g. chest, brain) that was scanned at least twice, capped at five locations per patient. 
This resulted in $\approx$158,000 samples for pretraining.

\paragraph{Pre-training for feature selection.}
We first validate our modeling assumptions from from Section \ref{sec:setup} using a linear model.
The goal of this section is to roughly validate our assumptions and the setup of our theoretical model.
We compare downstream progression extraction performance of (i) a vanilla logistic regression model and (ii) a logistic regression model only using the features selected by OCP.  We test on all five folds for five training dataset fractions.

For each experiment, our dataset is featurized using the unigrams and bigrams that occur in at least 5\% of the labeled training data set. They are vectorized using the \textit{term frequency-inverse document frequency} weighting scheme, via scikit-learn \citep[][BSD 3-clause license]{scikit-learn}. We conduct feature selection as an optional intermediate step preceding progression extraction. 
For OCP, we train a logistic regression model with L1 penalty over the 158,000 pre-training pairs of consecutive radiology notes. The regularization constant was set such that there were $50\pm5$ features with nonzero weights. 
In addition to OCP-derived features, we select the 50 most common features, and 5 random subsets of 50 features to serve as a comparison.

We train scikit-learn logistic regression models for downstream progression extraction over each feature set; further details are in Appendix \ref{apdx:experiments}. Results can be seen in Figure 3a. Even with a simple bag-of-words representation, feature selection with OCP outperforms \emph{directly} training a tuned logistic regression model on the available labeled data (``All features''), especially for small dataset sizes. A paired $t$-test finds that the model with OCP-selected features is significantly better than the direct-downstream model on a sixteenth of the data ($p<0.05$). Note that selecting the most common features or a random set of features does not compare, showing that OCP does not improve performance by simply reducing the feature dimension in a redundant space.

We manually examined the OCP-selected features and their coefficients (Figure \ref{fig:sub2}).
The features included (e.g. \textit{increased, decreased}) strongly indicate disease progression, while those discarded (e.g. \textit{discussed}) largely seem to be noise. 
Of the nonzero coefficients, 76\% have a positive weight; this indicates that the pre-training model focuses mostly on features that have been \textit{turned on} to conduct the ordering task, fitting with our motivating theoretical setting.

\begin{table*}[t]
\setlength{\tabcolsep}{4.5pt}
    \centering
        \caption{Performance of deep methods on cancer progression extraction. The first row contains the mean AUC of OCP $\pm$ its std dev. The following rows contain the mean AUC advantage of OCP over each comparison method, and the percentage of time OCP outperforms that method, across the 3 seeds and 5 folds.}
\begin{tabular}{l|ccccc|}
\cline{2-6}  & \multicolumn{5}{c|}{\textit{Fraction of training data}}     \\ \hline
\multicolumn{1}{|l|}{\textit{AUC diff. (OCP Win \%)}} & \multicolumn{1}{c}{\textbf{1}} & \multicolumn{1}{c}{\textbf{1/2}} & \multicolumn{1}{c}{\textbf{1/4}} & \multicolumn{1}{c}{\textbf{1/8}} & \multicolumn{1}{c|}{\textbf{1/16}} \\ \hline
\multicolumn{1}{|l|}{\textbf{OCP AUC}}   
& 0.87 $\pm$ .03    & 0.86 $\pm$ .04  &  0.84 $\pm$ .04 &  0.82 $\pm$ .03  & 0.81 $\pm$ .03        \\\hline  \cline{1-6}
\multicolumn{1}{|l|}{\textbf{OCP $-$ BERT}}   
& 0.08 (93\%)    & 0.12 (100\%)  & 0.12 (100\%) & 0.18 (100\%) & 0.22 (100\%)       \\\hline  
\multicolumn{1}{|l|}{\textbf{OCP $-$ FT LM}}   
& 0.03 (80\%)  & 0.04 (82\%) & 0.04 (82\%) &  0.08 (93\%) & 0.10 (89\%)     \\\hline  
\multicolumn{1}{|l|}{\textbf{OCP $-$ Pt-Contrastive}}   
& 0.03 (86\%)   & 0.03 (77\%) & 0.05 (91\%) & 0.09 (91\%) & 0.12 (97\%)  \\\hline  

\multicolumn{1}{|l|}{\textbf{OCP $-$ PCL}}   
& 0.00 (53\%)   & 0.00 (46\%) & 0.03 (64\%)  & 0.03 (76\%)   & 0.06 (87\%)   \\\hline  
\end{tabular}
\label{tab:msk_deep}
\end{table*}
\paragraph{Nonlinear representations.}
We now study the use of OCP for pre-training nonlinear representations.
We compare performance of a BERT model pre-trained using OCP to several other self-supervision methods. We investigate the BERT base model and the BERT base model after it is pre-trained using: (i) \textit{FT LM}: fine-tuned masked language modeling over an equivalent number of impressions, (ii)  \textit{Pt-Contrastive}: a patient-level contrastive objective (identical positive sampling to OCP and PCL, but each negative is a random note of the same note type from a \emph{different patient}, similar to \citet{diamant2021patient}), (iii) \textit{PCL}: contrastive pre-training with PCL sampling (each negative is a random pair of notes of the same type from the same patient), (iv) \textit{OCP}: contrastive pre-training with OCP sampling (each negative is a pair of notes of the same type in the incorrect order). 
All pre-training is conducted over three seeds, and all three contrastive objectives were trained with the same number of pairs (158,000). Implementation for 
language modeling and contrastive pre-training came from \citet[][Apache-2.0 License]{wolf-etal-2020-transformers} with full details in Appendix \ref{apdx:experiments}. 
After model pre-training/fine-tuning, the self-supervised representation layers were frozen, and a single L2-regularized linear layer was added on top. 
The goal of freezing was to isolate the effect of pre-training to understand representation quality, due to the instability of training BERT on small downstream tasks \citep{zhang2020revisiting}.

Results can be seen in Table 1. The top row shows that OCP has only a modest drop in performance even when trained on the data from just 5 patients ($\frac{1}{16}$). 
Since correlations exist in AUC across the 5 folds, 3 seeds, and 5 dataset sizes, standard statistical comparison testing is inappropriate. 
Instead, we present the mean increase in AUC from OCP, as well as the percentage of the time OCP outperformed the comparisons. Unsurprisingly, BERT alone (trained on non-clinical text) unsurprisingly does not perform well out-of-the-box; fine-tuning with language modeling improves performance, but still suffers in the low data regime. Among the contrastive objectives, the cross-patient objective is the weakest, which may follow since its pre-training task was the easiest (82\% accuracy on validation). It could rely on features that differed between patients, instead of being forced to focus on the temporal features that differed \emph{within} a patient's timeline. 
PCL performs equivalently to OCP at large data sizes, but at the smaller data set sizes, it loses to OCP a large majority of the time.

\section{Related work}
\label{sec:related}

\paragraph{Order pre-training.}

Others have found order-based self-supervision useful for more complex time-series data, but without theoretical study.
For example, learning the order of frames within a video yields representations useful for downstream activity classification \citep{fernando2015modeling, misra2016shuffle, lee2017unsupervised, wei2018learning}.
Most similar to our work, \citet{hyvarinen2017nonlinear} introduced \emph{permutation-contrastive learning} (PCL) and proved \emph{nonlinear} identifiability for representations learned using PCL in an ICA setting. 
That is, they gave distributional conditions where PCL provably recovers the ``correct'' nonlinear representation of the input given infinite unlabeled data.
Our theoretical and empirical results indicate that there can be nontrivial differences between OCP and PCL's downstream performance.
Deeper understanding of what data distributions and downstream tasks are ``right'' for PCL versus OCP (and for other contrastive sampling methods) is an interesting direction for future theoretical study.

\vspace{-0.5\baselineskip}
\paragraph{Pre-training for medical time-series.}
Several other pre-training objectives have been explored on clinical time-series data. 
A contrastive learning setup similar to our \emph{patient-contrastive} baseline has shown promising results on electrocardiograph signals \citep{diamant2021patient, kiyasseh2020clocs}. \citet{banville2021uncovering} studied a contrastive objective for electroencephalography signals, in which windows of a signal are judged to be similar if they occur within a certain time gap, and dissimilar if they are far away in time.
Intuitively, this objective is well-suited to data where the true representation ``changes slowly'' with time, as with \citet{franceschi2019} (discussed in Section \ref{sec:intro}). Other objectives include auto-encoding \citep{fox2019learning} and masked prediction over text and tabular data, to mixed results
\citep{steinberg2020language, huang2019clinicalbert, yoon2020vime, mcdermott2021pre}. 
Multi-task pre-training supplies improvements, but unlike our work, it relies on additional labeled data from closely-related downstream tasks \citep{mcdermott2021pre}.

\vspace{-0.7\baselineskip}
\paragraph{Self-supervision theory.}
Like our work, \citet{arora2019theoretical, liu2021contrastive} and \citet{tosh2020contrastive, tosh2021contrastive} give downstream finite-sample error bounds for representations learned using particular contrastive learning objectives.
Our motivating theoretical setting and proof techniques are simpler than the ones considered in these works, but we show that our setup in Section \ref{sec:setup} is (i) complex enough to allow for some of the same nontrivial behavior observed by contrastive methods in practice (Sections \ref{sec:pcl-diff}, \ref{subsec:synth-data}) and (ii) it has some practical applications (Section \ref{subsec:real-data}).

\section{Limitations and Conclusion}
\label{sec:conclusion}
We have shown both theoretically and empirically that order-contrastive pre-training is an effective self-supervised method for certain types of time-series data and downstream tasks. 
On real-world longitudinal health data, we find that representations from OCP significantly outperform others, including the similar PCL, in the small data regime. Concretely, being able to structure variables from longitudinal, label-scarce, data in health records could enable us to evaluate large scale retrospective datasets and potentially inform future clinical trials and patient care. 

However, OCP is not always suitable.
For example, cases of temporal leakage (e.g., the date in a note) can lead to weak OCP (and PCL) representations downstream, since they provide a shortcut during pre-training. 
While dates are straightforward to censor, more complex global nonstationarities irrelevant to downstream tasks would present a challenge to these methods. We additionally wish to emphasize that not all clinical tasks have time-irreversible expressions in the data (e.g., acute/temporary conditions, such as pregnancy), so the motivating model assumptions (particularly Assumption \ref{asmp:S-suitable}) should be considered before applying OCP. 

Our theoretical setup and results in Section \ref{sec:setup} also serve to highlight that contrastive pre-training methods can be very sensitive to the precise sampling details, and provide a simple model for studying these details that is still complex enough to capture some empirical phenomena.
This suggests that obtaining broader theoretical guidelines for selecting a contrastive distribution is an interesting direction for future work.

\subsubsection*{Acknowledgements}
We would like to thank Rebecca Boiarsky for helpful comments on the manuscript, Richard Do for advising as a subject matter expert on radiology reports, and the Institute of Advanced Study for their hospitality in hosting MA.
DS and HL were partially supported by AitF award CCF-1723344, and MA by a Takeda Fellowship. 
The authors acknowledge funding for Memorial Sloan Kettering Cancer Center received through the NIH/NCI institutional P30 CA008748 grant.




\clearpage

\bibliography{references}

\clearpage
\appendix

\thispagestyle{empty}

\onecolumn \makesupplementtitle
\section{Theory details}
\label{apdx:theory}
In this section, we provide proofs for the simple model introduced in Section \ref{sec:setup}.
In particular, we prove the finite sample bound for OCP by first proving that the driver features $S$ are the \emph{unique optimal representation}, and then we apply standard uniform convergence arguments to obtain the bound \eqref{eqn:pretrain-bound}.

We also show how so-called \emph{bias} \citep{chuang2020debiased} in the distribution of negative samples can affect OCP.
In particular, suppose that instead of choosing windows $(W, W+1)$ as positive and $(W+1, W)$ as negatives, we chose $(W,W+1)$ as positives and a random pair $(W,W')$ with $|W-W'| = 1$ as negatives.
This is analogous to PCL's negative sampling, where some of the negatives are still in the correct order. 
The difference with PCL is just that all negatives are still consecutive.
For the model in Section \ref{sec:setup}, we prove that (i) in the infinite-data limit, this ``biased'' version recovers the same representation as $S$ (so in fact, the estimator obtained by minimizing this objective is \emph{not} biased in a statistical sense) and (ii) the bound on \emph{unlabeled} sample complexity required to find a good representation is worse for this biased version of OCP than for the unbiased version.
This gives more theoretical evidence for the value of \emph{de-biasing} the negative distribution in contrastive learning, where possible: \emph{even if it doesn't change the representation learned with infinite unlabeled data, de-biasing the contrastive distribution can improve unlabeled sample complexity.}

\subsection{Assumptions and example class of distributions}
Here we provide simple assumptions under which a set of features $S$ is the unique optimal solution to \eqref{eqn:pretraining-obj} for the model from Section \ref{sec:setup}.
For simplicity in this section, we only consider consecutive windows and the case where $\tau$ (the sequence length) is the same for all sequences, but all of our results generalize (with suitable modifications to these assumptions) to the case where the positive and negative distributions over windows are symmetric up to ordering of the elements (correct versus incorrect), and to the case with a distribution over $\tau$.\footnote{For example, if for all $T$ and $t \in \{0,\ldots, T-1\}$, $\P[X^t, X^{t+1} | \tau = T] = \P[X^t, X^{t+1} | \tau \ge t+1]$ (i.e., the distribution of any window pair only depends on the fact that the trajectory is still active, and not on the actual length) then assumptions remain the same.}

\begin{assumption}
\label{asmp:s-stays-active-apdx}
For all $i\in S$ and all $t$,
$\P[X^t_i = 1, X_i^{t+1} = 0] = 0.$
\end{assumption}

\begin{assumption}
\label{asmp:time-reversible-apdx}
For all $v \in \{0,1\}^d$ and $v' \in \{0,1\}^d$ with $v_S = v'_S$, and for all $t$, $\P[X^t = v, X^{t+1} = v'] = \P[X^t = v', X^{t+1} = v].$
\end{assumption}
\begin{assumption}
  \label{asmp:U-asmp-apdx}
  For all $U$ such that $S\not\subset U$, there exists $t$ and $v\in \{0,1\}^d$, $v'\in \{0,1\}^d$, with $v_{U\cap S} = v'_{U\cap S}$ and:
  \[
  \min(\P[X^t_S \subsetneq X^{t+1}_S, X^t_U = v_U, X^{t+1}_U = v'_U]; \P[X^t_S \subsetneq X^{t+1}_S, X^t_U = v'_U, X^{t+1}_U = v_U]) > 0.
  \]
\end{assumption}
\begin{assumption}
\label{asmp:S-suitable-apdx}
The features $S$ are suitable for the downstream classification task:
    \[
    \argmin_{f \in \cF} \P[f(X_S) \ne Y] = \argmin_{f \in \cF, g\in \cG} \P[(f\circ g)(X) \ne Y]
    \]
\end{assumption}

Now we prove that the example from Section \ref{sec:setup} satisfies these assumptions (in particular, Assumptions \ref{asmp:s-stays-active-apdx}-\ref{asmp:U-asmp-apdx}).

\begin{theorem}
\label{thm:asmp-example}
Partition the indices $[d]$ into three sets $(S, \hat{S}, B)$ satisfying the following assumptions:
\begin{itemize}[beginpenalty=10000]
    \item Time-irreversible features $S$ satisfying Assumption \ref{asmp:s-stays-active}. We also assume that each $i\in S$ has a nonzero probability of activating on its own, without the other features in $S$.
    That is, for each $i\in S$ there exists $t\in [\tau-1]$, with $\P[X^t_i = 0, X^{t+1}_i = 1, X^t_{S \setminus \{i\}} = X^{t+1}_{S \setminus \{i\}}] > 0$.
    \item Noisy versions $\hat{S}$ of $S$: for each $j\in \hat{S}$, there exists $i \in S$ with $\P[X^t_j = X^t_i] = (1-\epsilon_i)$, with $\epsilon_i > 0$, for all $t$. Additionally, $X^t_i$ is conditionally independent of the other variables (for all times) given its driver $X^t_i$: $\P[X^t_j, X^1_{\mathcal{I}_1},\ldots, X^\tau_{\mathcal{I}_{\tau}} | X^t_i] = \P[X^t_j| X^t_i]\P[X^1_{\mathcal{I}_1},\ldots, X^\tau_{\mathcal{I}_{\tau}} | X^t_i]$, where $\mathcal{I}_{\hat{t}} = [d]\setminus j$ if $\hat{t} = t$ and $[d]$ otherwise.
    \fxnote{this is not precisely the right assumption; be more careful notationally to include all times}
    \item Background, reversible features $B$: features such that for all $t$ and all $v,v' \in \{0,1\}^d$, 
    \begin{align*}
    &\P\left[(X^t_B, X^t_{[d]\setminus B}) = (v_B, v_{[d]\setminus B}), (X^{t+1}_B, X^{t+1}_{[d]\setminus B}) = (v'_B, v'_{[d]\setminus B})\right] = \\
    &\P\left[(X^t_B, X^t_{[d]\setminus B}) = (v'_B, v_{[d]\setminus B}), (X^{t+1}_B, X^{t+1}_{[d]\setminus B}) = (v_B, v'_{[d]\setminus B})\right]
    \end{align*}
\end{itemize}
This class of distributions satisfies Assumption \ref{asmp:s-stays-active-apdx}-\ref{asmp:U-asmp-apdx}.
\end{theorem}

\begin{proof}
Assumption \ref{asmp:s-stays-active-apdx} is satisfied by definition.
For Assumption \ref{asmp:time-reversible-apdx}, fix $v, v' \in \{0,1\}^d$ with $v_S = v'_S$ and $t \in [\tau-1]$.
Then we have:
\[
\P[X^t = v, X^{t+1} = v'] = \P[X^t_S = v_S, X^{t+1}_S = v'_S, X^t_{\hat{S}} = v_{\hat{S}}, X^{t+1}_{\hat{S}} = v'_{\hat{S}}, X^t_B = v_B, X^{t+1}_B = v'_B].
\]
Using the conditional independence assumption for the features in $\hat{S}$, the right-hand-side factors to:
\[
\P[X^t_S = v_S, X^{t+1}_S = v'_S, X^t_B = v_B, X^{t+1}_B = v'_B]\P[X^t_{\hat{S}} = v_{\hat{S}}|X^t_S=v_S]\P[X^{t+1}_{\hat{S}} = v'_{\hat{S}}|X^{t+1}_S=v'_S].
\]
Because $v_S = v'_S$, trivially this is equal to:
\[
\P[X^t_S = v'_S, X^{t+1}_S = v_S, X^t_B = v_B, X^{t+1}_B = v'_B]\P[X^t_{\hat{S}} = v_{\hat{S}}|X^t_S=v_S]\P[X^{t+1}_{\hat{S}} = v'_{\hat{S}}|X^{t+1}_S=v'_S].
\]
The reversibility of the features in $B$ (and summing over the full joint distribution) imply that the above is equal to:
\[
\P[X^t_S = v'_S, X^{t+1}_S = v_S, X^t_B = v'_B, X^{t+1}_B = v_B]\P[X^t_{\hat{S}} = v_{\hat{S}}|X^t_S=v_S]\P[X^{t+1}_{\hat{S}} = v'_{\hat{S}}|X^{t+1}_S=v'_S].
\]
Finally, note that because $v_S = v'_S$, $X^t_{\hat{S}}$ and $X^{t+1}_{\hat{S}}$ are identically distributed by the definition of features in $\hat{S}$, so:
\begin{align*}
\P[X^t_{\hat{S}} = v_{\hat{S}}|X^t_S=v_S]\P[X^{t+1}_{\hat{S}} = v'_{\hat{S}}|X^{t+1}_S=v'_S] = \\
\P[X^t_{\hat{S}} = v'_{\hat{S}}|X^t_S=v'_S]\P[X^{t+1}_{\hat{S}} = v_{\hat{S}}|X^{t+1}_S=v_S].
\end{align*}
Combining with the previous equation and simplifying, we obtain
\[
\P[X^t = v, X^{t+1} = v'] = \P[X^t = v', X^{t+1} = v],
\]
which is Assumption \ref{asmp:time-reversible-apdx}.

Finally, for Assumption \ref{asmp:U-asmp-apdx}, fix $U\subset [d]$ with $S \not \subset U$, and fix some $i \in S\setminus U$.
We assumed that each $i \in S$ has some probability of activating on its own, so there exists $t$ with $\P[X^t_i = 0, X^{t+1}_i = 1, X^t_{S\setminus \{i\}} = X^{t+1}_{S\setminus \{i\}}] > 0$.
In particular, this implies that there exist $v, v'$ with $v_i = 0$, $v'_i = 1$, $v_{S\setminus \{i\}} = v'_{S\setminus \{i\}}$ and $\P[X^t = v, X^{t+1}= v'] > 0$.
By summing over variables in this joint distribution, that implies $\P[X^t_i = 0, X^{t+1}_i = 1, X^t_{U} = v_U, X^{t+1}_U = v'_U] > 0$, so $\P[X^t_S \subsetneq X^{t+1}_S, X^t_{U} = v_U, X^{t+1}_U = v'_U] > 0.$
Now we need to reverse $v$ and $v'$ for the $U$ indices to account for the other term in the $\min$ of Assumption \ref{asmp:U-asmp-apdx}.
For each $j \in \hat{S}$, we can take $v_j = v'_j = 0$ while maintaining $\P[X^t = v, X^{t+1} = v'] > 0$ and $\P[X^t_S \subsetneq X^{t+1}_S, X^t_{U} = v_U, X^{t+1}_U = v'_U] > 0$, since we assumed that all values of $\epsilon_j > 0$.
We know that $\P[X^t = v, X^{t+1} = v'] > 0$, so we will try to rearrange $\P[X^t_U = v'_U, X^{t+1}_U = v_U, X^t_{S} \subsetneq X^{t+1}_S]$ to make use of this fact by attempting to switch $v$ for $v'$.

We have: \[
\P[X^t_U = v'_U, X^{t+1}_U = v_U, X^t_{S} \subsetneq X^{t+1}_{S}] \ge \P[X^t_U = v'_U, X^{t+1}_U = v_U, X^t_{S} = v_{S}, X^{t+1}_{S} = v'_{S}],
\]
because $v_i =0,\ v'_i = 1$. 
Let $\mathcal{J}$ be the set of indices in $\hat{S}$ corresponding to noisy indicators of feature $i$.
Because we chose $v,v'$ so that $v_{S\setminus\{i\}} = v'_{S\setminus \{i\}}$, all features in $\hat{S}\setminus \mathcal{J}$ are identically distributed at times $t$ and $t+1$.
Therefore, we can switch $v$ and $v'$ for some of the indices in $U$:
\begin{align*}
\P[X^t_U = v'_U, X^{t+1}_U = v_U, X^t_{S\setminus U} = v_{S\setminus U}, X^{t+1}_{S\setminus U} = v'_{S\setminus U}] = \P[&X^t_{U\cap S} = v_{U\cap S}, X^{t+1}_{U \cap S} = v'_U,\\
&X^t_{U\cap \hat{S} \setminus \mathcal{J}} = v_{U\cap \hat{S} \setminus \mathcal{J}}, X^{t+1}_{U\cap \hat{S} \setminus \mathcal{J}} = v'_{U\cap \hat{S} \setminus \mathcal{J}},\\
&X^t_{U\cap B} = v'_{U\cap B}, X^{t+1}_{U\cap B} = v_{U\cap B},\\
&X^t_{U\cap \mathcal{J}} = v'_{U\cap \mathcal{J}}, X^{t+1}_{U\cap \mathcal{J}} = v_{U\cap \mathcal{J}},\\
&X^t_{S\setminus U} = v_{S\setminus U}, X^{t+1}_{S\setminus U} = v'_{S\setminus U}]
\end{align*}
By the definition of the features in $B$, we can also switch $v$ and $v'$ for $U\cap B$.
The only remaining terms to switch are $X^t_{U\cap \mathcal{J}}$ and $X^{t+1}_{U\cap \mathcal{J}}$. 
Because we took $v_j = v'_j = 0$ for all $j \in \mathcal{J}$ without loss of generality, and $v_i = 0$, $v'_i = 1$, we have:
\begin{align*}
\P[X^t_{\mathcal{J}} = v'_{\mathcal{J}} | X^t_i = v_i = 0] = \prod_{j\in     \mathcal{J}}(1-\epsilon_j) = \P[X^t_{\mathcal{J}} = v_{\mathcal{J}} | X^t_i = v_i = 0]\\
\P[X^{t+1}_{\mathcal{J}} = v_{\mathcal{J}} | X^{t+1}_i = v'_i = 1] = \prod_{j\in     \mathcal{J}}\epsilon_j = \P[X^{t+1}_{\mathcal{J}} = v'_{\mathcal{J}} | X^{t+1}_i = v'_i = 1].
\end{align*}
Hence, we can finally combine:
\begin{align*}
\P[X^t_U = v'_U, X^{t+1}_U = v_U, X^t_{S\setminus U} = v_{S\setminus U}, X^{t+1}_{S\setminus U} = v'_{S\setminus U}] = \P[&X^t_{U\cap S} = v_{U\cap S}, X^{t+1}_{U \cap S} = v'_U,\\
&X^t_{U\cap \hat{S} \setminus \mathcal{J}} = v_{U\cap \hat{S} \setminus \mathcal{J}}, X^{t+1}_{U\cap \hat{S} \setminus \mathcal{J}} = v'_{U\cap \hat{S} \setminus \mathcal{J}},\\
&X^t_{U\cap B} = v_{U\cap B}, X^{t+1}_{U\cap B} = v'_{U\cap B},\\
&X^t_{U\cap \mathcal{J}} = v_{U\cap \mathcal{J}}, X^{t+1}_{U\cap \mathcal{J}} = v'_{U\cap \mathcal{J}},\\
&X^t_{S\setminus U} = v_{S\setminus U}, X^{t+1}_{S\setminus U} = v'_{S\setminus U}]
\end{align*}
So we have shown for these $v,v'$ that: 
\begin{align*}
\P[X^t_U = v'_U, X^{t+1}_U = v_U, X^t_{S\setminus U} = v_{S\setminus U}, X^{t+1}_{S\setminus U} = v'_{S\setminus U}] = \P[&X^t_U = v_U, X^{t+1}_U = v'_U, \\
&X^t_{S\setminus U} = v_{S\setminus U}, X^{t+1}_{S\setminus U} = v'_{S\setminus U}].
\end{align*}
But we know the RHS is positive because $v,v'$ had $\P[X^t = v, X^{t+1} = v'] > 0$.
Therefore,
\begin{align*}
\P[X^t_U = v'_U, X^{t+1}_U = v_U, X^t_S \subsetneq X^{t+1}_{S}] &\ge \P[X^t_U = v'_U, X^{t+1}_U = v_U, X^t_{S\setminus U} = v_{S\setminus U}, X^{t+1}_{S\setminus U} = v_{S\setminus U}]\\
&> 0.
\end{align*}
So we have shown that both 
$\P[X^t_S \subsetneq X^{t+1}_S, X^t_U = v_U, X^{t+1}_U = v'_U]$ and $\P[X^t_S \subsetneq X^{t+1}_S, X^t_U = v'_U, X^{t+1}_U = v_U]$ are strictly positive, which gives Assumption \ref{asmp:U-asmp-apdx}.
\end{proof}

\subsection{Optimal representations and finite-sample guarantee}
This first lemma shows that the feature set $S$ is \emph{an} optimal representation for the OCP objective.
\begin{lemma}
\label{lem:S-optimal}
Let \[
h^*(v,v') = \argmax_{y\in\{-1,1\}} \P[Y=y | X^W = v, X^{W'} = v']
\]
be the Bayes-optimal classifier for the OCP objective.
There exists a classifier $h(v_S,v'_S)$, depending only on the $S$ coordinates of $v$ and $v'$, achieving the same OCP error as $h^*$.
\end{lemma}
\begin{proof}
First we compute the error of the Bayes-optimal classifier.
For $A \subset [d]$, $v,v' \in \{0,1\}^d$, we say $v_A \subsetneq v'_A$ if $\{i \in A : v_i = 1\} \subsetneq \{i \in A : v'_i = 1\}$.
Assumption \ref{asmp:s-stays-active} implies that:
\[
\P[Y=1 | X^W, X^{W'}] = \begin{cases}
    1 & X^W_S \subsetneq X^{W'}_S\\
    0 & X^{W'}_S \subsetneq X^{W}_S.
    \end{cases}
\]
That is, given windows $(X^W, X^{W'})$, if the set of $S$ variables active in $X^W$ is properly contained in the set of $S$ variables active in $X^{W'}$, we know the windows must be in the correct order.
Likewise, if the active $S$ variables in $X^{W'}$ are properly contained in those of $X^W$, the windows must be in the wrong order.
Assumption \ref{asmp:s-stays-active} also implies that the only possible cases are $X^W_S \subsetneq X^{W'}_S$, $X^{W'}_S \subsetneq X^{W}_S$, and $X^W_S = X^{W'}_S$.
For the last case, we have:
\[
\P[Y = 1 | X^W, X^{W'}] = \frac{\P[X^W, X^{W'} | Y=1]}{\P[X^W, X^{W'} | Y=1] + \P[X^W, X^{W'} | Y=-1]}
\]
Assumption \ref{asmp:time-reversible} implies that the two terms in the denominator are equal when $X^W_S = X^{W'}_S$.
This implies that when $X^W_S = X^{W'}_S$, $\P[Y = 1 |X^W, X^{W'}] = 1/2$, so the Bayes error is:
\[
\P[Y \ne h^*(X^W, X^{W'})] = \frac{1}{2}\P[X^W_S = X^{W'}_S].
\]
Hence, we can define:
\begin{equation}
    \label{eqn:h-defn}
h(v_S, v'_S) = \begin{cases}
1 & v_S \subsetneq v'_S\\
-1 & \mbox{otherwise},
\end{cases}
\end{equation}
where we've made the arbitrary choice of $-1$ when $v_S = v'_S$.
Then we have:
\[
\P[Y \ne h(X^W, X^{W'})] = \frac{1}{2}\P[X^W_S = X^{W'}_S].
\]
So we can achieve the Bayes error using only the coordinates in $S$.
Additionally, note that $h(v_S, v'_S)$ can easily be written as a linear function over $2|S|$ coordinates.
If we let $z = (v_{S_1},\ldots,v_{S_{|S|}},v'_{S_1},\ldots,v'_{S_{|S|}})$ (i.e., the first $|S|$ coordinates represent $v_S$ and the second $|S|$ $v'_S$), then $h(v_S, v'_S) = h(z) = \operatorname{Sign}(\sum_{i=|S|+1}^{2|S|}z_i - \sum_{i=1}^{|S|}z_i),$ where we break ties with $-1$ in accordance with \eqref{eqn:h-defn}.
This means that we can safely take $\cC$ to be the set of linear functions (in fact, we've just shown one linear function will always suffice), so $\VC(\cC)$ is bounded.
\end{proof}

Recall that we are searching over hypotheses $h = (c,g)$ with $h(v,v') = h(g(v), g(v'))$. When $g$ is a feature selector (i.e., $g$ picks a certain set of features) and $c$ is unconstrained, the previous lemma immediately implies we can take $g = S$.
\begin{corollary}
The feature selector $g$ that selects the features $S$ is an optimal representation for the OCP objective.
\end{corollary}

Now we show that $S$ is the \emph{unique} optimal choice of features.
This lemma gives a clean expression for the error incurred by choosing a set of features $U$ potentially different from $S$.
\begin{lemma}
\label{lemma:S-unique}
For $U \subset [d]$, define the error of $U$ as the best OCP loss achievable when using the features in the set $U$:
\[
\err(U) = \inf_{c \text{ measurable}} \P[Y \ne c(X^W_U, X^{W'}_U)].
\]
For any $p, q \in \{0,1\}^{|U|}$ with $p_{U\cap S} = q_{U\cap S}$, define:
\[
m_U(p,q) = \min\left(\P[X^W_S \subsetneq X^{W'}_S | X^W_U = p, X^{W'}_U=q]; \P[X^{W'}_S \subsetneq X^{W}_S | X^W_U = p, X^{W'}_U=q]\right).
\]
The expected value of $m_U(p,q)$ is given by:
\[
\E_{p,q}[m_U(p,q)] = \sum_{p,q \in \{0,1\}^{|U|}}\mathbb{I}[p_{U\cap S} = q_{U\cap S}]\P[X^W_U=p, X^{W'}_U = q]m_U(p,q)
\]
Then for any $U$,
\begin{align*}
\err(U) &= \frac{1}{2}\P[X^W_S = X^{W'}_S] + \E_{p,q}[m_U(p,q)]\\
&= \err(S) + \E_{p,q}[m_U(p,q)]
\end{align*}
\end{lemma}
\begin{proof}
As with $S$, we know the optimal classifier using the features $U$ is given by:
\[
h_U(p,q) = \argmax_{y} \P[Y = y | X^W_U = p, X^{W'}_U = q].
\]

So we have:
\begin{align}
\label{eqn:big-error-expansion}
\err(U) &= \P[Y \ne h_U(p,q)] = \sum_{p,q} \P[Y \ne h_U(p,q) | X^W_U=p, X^{W'}_U=q]\P[X^W_U=p, X^{W'}_U=q]\nonumber\\
    & = \sum_{p,q} \P[Y \ne h_U(p,q), X^W_S = X^{W'}_S| X^W_U=p, X^{W'}_U=q]\P[X^W_U=p, X^{W'}_U=q]\nonumber\\
    & + \sum_{p,q} \P[Y \ne h_U(p,q), X^W_S \subsetneq X^{W'}_S| X^W_U=p, X^{W'}_U=q]\P[X^W_U=p, X^{W'}_U=q]\nonumber\\
    & + \sum_{p,q} \P[Y \ne h_U(p,q), X^{W'}_S \subsetneq X^{W}_S| X^W_U=p, X^{W'}_U=q]\P[X^W_U=p, X^{W'}_U=q],
\end{align}
where we've used Assumption \ref{asmp:s-stays-active} to narrow down to these three cases for the relationship between $X^W_S$ and $X^{W'}_S$.
The first term of \eqref{eqn:big-error-expansion} is the easiest to handle:
\begin{align*}
& \sum_{p,q} \P[Y \ne h_U(p,q), X^W_S = X^{W'}_S| X^W_U=p, X^{W'}_U=q]\P[X^W_U=p, X^{W'}_U=q]\\
&= \sum_{p,q} \P[Y \ne h_U(p,q) | X^W_S = X^{W'}_S]\P[X^W_S = X^{W'}_S | X^W_U=p, X^{W'}_U=q]\P[X^W_U=p, X^{W'}_U=q]\\
&= \frac{1}{2} \sum_{p,q}\P[X^W_S = X^{W'}_S | X^W_U=p, X^{W'}_U=q]\P[X^W_U=p, X^{W'}_U=q]\\
&= \frac{1}{2}\P[X^W_S = X^{W'}_S].
\end{align*}
In the first equality, we used that only the values of $X^W_S$ and $X^{W'}_S$ affect the posterior distribution of $Y$, as we saw in the proof of Lemma \ref{lem:S-optimal}.
The second equality used that $\P[Y \ne h_U(p,q) | X^W_S = X^{W'}_S] = \frac{1}{2}$ regardless of the value of $h_U(p,q)$, since we showed in the proof of Lemma \ref{lem:S-optimal} that $\P[Y=1 | X^W_S = X^{W'}_S] = \frac{1}{2}$.

Now we consider the second term of \eqref{eqn:big-error-expansion}:
\begin{align*}
    &\sum_{p,q} \P[Y \ne h_U(p,q), X^W_S \subsetneq X^{W'}_S| X^W_U=p, X^{W'}_U=q]\P[X^W_U=p, X^{W'}_U=q]\\
    &=\sum_{p,q} \P[Y \ne h_U(p,q) | X^W_S \subsetneq X^{W'}_S]\P[X^W_S \subsetneq X^{W'}_S | X^W_U=p, X^{W'}_U=q]\P[X^W_U=p, X^{W'}_U=q]\\
\end{align*}
When $p_{U\cap S} \subsetneq q_{U\cap S}$, it can be easily verified that $h_U(p,q) = 1$ and $Y=1$.
Likewise, when $q_{U\cap S} \subsetneq p_{U\cap S}$, $h_U(p,q) = Y = 0$.
So this term is equal to:
\[
\sum_{p,q\ :\ p_{U\cap S} = q_{U\cap S}} \P[Y \ne h_U(p,q) | X^W_S \subsetneq X^{W'}_S]\P[X^W_S \subsetneq X^{W'}_S | X^W_U=p, X^{W'}_U=q]\P[X^W_U=p, X^{W'}_U=q]
\]

Applying the same trick to the third term of \eqref{eqn:big-error-expansion} yields:
\begin{align*}
&\sum_{p,q} \P[Y \ne h_U(p,q), X^{W'}_S \subsetneq X^{W}_S| X^W_U=p, X^{W'}_U=q]\P[X^W_U=p, X^{W'}_U=q]\\
&= \sum_{p,q\ :\ p_{U\cap S} = q_{U\cap S}} \P[Y \ne h_U(p,q) | X^{W'}_S \subsetneq X^{W}_S]\P[X^{W'}_S \subsetneq X^{W}_S | X^W_U=p, X^{W'}_U=q]\P[X^W_U=p, X^{W'}_U=q]
\end{align*}
So we're left with:
\begin{align*}
&\sum_{p,q\ :\ p_{U\cap S} = q_{U\cap S}} \P[Y \ne h_U(p,q) | X^W_S \subsetneq X^{W'}_S]\P[X^W_S \subsetneq X^{W'}_S | X^W_U=p, X^{W'}_U=q]\P[X^W_U=p, X^{W'}_U=q] +\\
&\sum_{p,q\ :\ p_{U\cap S} = q_{U\cap S}} \P[Y \ne h_U(p,q) | X^{W'}_S \subsetneq X^{W}_S]\P[X^{W'}_S \subsetneq X^{W}_S | X^W_U=p, X^{W'}_U=q]\P[X^W_U=p, X^{W'}_U=q]
\end{align*}
Now consider a fixed $p,q$ with $p_{U\cap S} = q_{U\cap S}$.
If $X^W_S \subsetneq X^{W'}_S$, then $Y=1$, so we only pay if $h_U(p,q) = -1$.
Likewise, if $X^{W'}_S \subsetneq X^{W}_S$, $Y=-1$, so we only pay if $h_U(p,q) = 1$.
Since $h_U(p,q)$ is the Bayes-optimal classifier, it pays the minimal loss for each $p,q$.
So the above sum is equal to:
\[
\sum_{p,q\ :\ p_{U\cap S} = q_{U\cap S}}\min(\P[X^W_S \subsetneq X^{W'}_S | X^W_U=p, X^{W'}_U=q], \P[X^{W'}_S \subsetneq X^{W}_S | X^W_U=p, X^{W'}_U=q])\P[X^W_U=p, X^{W'}_U=q],
\]
which is precisely $\E_{p,q}[m_U(p,q)]$.
\end{proof}
\begin{lemma}
\label{lem:m-positive}
For any $U \subset [d]$ with $S\not\subset U$, $\E_{p,q}[m_U(p,q)] > 0$.
\end{lemma}
\begin{proof}
Fix $U \subset [d]$ with $S \not \subset U$. 
Assumption \ref{asmp:U-asmp} states that there exists $t$ and $p,q$ with $p_{U\cap S} = q_{U\cap S}$ satisfying:
\begin{align}
&\P[X^t_S \subsetneq X^{t+1}_S, X^t_U = p, X^{t+1}_U = q] > 0\label{eqn:a3-first-part}\\
&\P[X^t_S \subsetneq X^{t+1}_S, X^t_U = q, X^{t+1}_U = p] > 0\label{eqn:a3-second-part}
\end{align}
Fix $t, p, q$ from this assumption. Observe that:
\begin{align}
\label{eqn:apdx-S-unique-display}
\P[X_S^W \subsetneq X_S^{W'} | &X^W_U = p, X^{W'}_U = q] \ge \P[X_S^W \subsetneq X_S^{W'}, W=t, W'=t+1| X^W_U = p, X^{W'}_U = q]\nonumber\\
&= \P[X_S^W \subsetneq X_S^{W'} | W=t, W'=t+1 X^W_U = p, X^{W'}_U = q]\P[W=t, W'=t+1| X^W_U = p, X^{W'}_U = q]\nonumber\\
&= \P[X_S^t \subsetneq X_S^{t+1} | X^t_U = p, X^{t+1}_U = q]\P[W=t,W'=t+1 | X^W_U = p, X^{W'}_U = q]
\end{align}
The first term is strictly positive because \eqref{eqn:a3-first-part} guaranteed $\P[X^t_S \subsetneq X^{t+1}_S, X^t_U = p, X^{t+1}_U = q] > 0$.
\begin{align*}
\P[W=t,W'=t+1|X^W_U=p,X^{W'}_U=q] &= \frac{\P[X^W_U=p,X^{W'}_U=q| W=t,W'=t+1]\P[W=t,W'=t+1]}{\P[X^W_U=p,X^{W'}_U=q]}\\
&= \frac{\P[X^t_U=p,X^{t+1}_U=q]\P[W=t, W'=t+1]}{\P[X^W_U=p,X^{W'}_U=q]}\\
&= \frac{\P[X^t_U=p,X^{t+1}_U=q]\P[W=t, W'=t+1]}{\sum_{w,w'}\P[X^w_U=p,X^{w'}_U=q | W=w, W'=w']\P[W=w,W'=w']}
\end{align*}
Equation \eqref{eqn:a3-first-part} guarantees the first term of the numerator is positive.
The second term of the numerator is positive by the definition of OCP sampling.
The denominator is positive since the numerator is a term in the denominator.
Hence, $\P[W=t,W'=t+1|X^W_U=p,X^{W'}_U=q] > 0$, so \eqref{eqn:apdx-S-unique-display} implies $\P[X_S^W \subsetneq X_S^{W'} | X^W_U = p, X^{W'}_U = q] > 0$.
A symmetric analysis setting $W' = t$ and $W = t+1$ and appealing to \eqref{eqn:a3-second-part} implies $\P[X_S^{W'} \subsetneq X_S^{W} | X^W_U = p, X^{W'}_U = q] > 0$.

We've shown that for any $U$ with $S\not\subset U$ there exists $p$, $q$ with $p_{U\cap S} = q_{U\cap S}$, $m_U(p,q) > 0$, and $\P[X^W_U = p, X^{W'}_U = q] > 0$. This implies that $\E_{p,q}[m(p,q)] > 0$.
\end{proof}

Combining the previous two lemmas immediately implies that $S$ is the unique optimal set of features for OCP (of size $|S|$).
\begin{corollary}
$S$ is the \emph{unique} optimal set of features of size $|S|$.
\end{corollary}

\paragraph{Notation.} 
Recall that the OCP pretraining objective $R_{ord}$ is defined over pairs $(c,g)$ of an order classifier $c$ with a representation $g$. For $h = (c,g)$, we equate $R_{ord}(h) = R_{ord}(g,c)$.
For a fixed $g$, we let $R_{ord}(g)$ be the ordering loss of the best classifier $c$ on top of $g$:
$R_{ord}(g) = \inf_{c\in \cC}R_{ord}(c,g)$.
Similarly, the downstream loss $R(f,g)$ is a function of a downstream classifier $f$ on top of a representation $g$.
We refer to the empirical versions of $R_{ord}$ and $R$ as $\hat{R}_{ord}$ and $\hat{R}$, respectively.
Recall from the proof of Lemma \ref{lem:S-optimal} that in our case, we can take $\cC$ to consist of one (linear) function without loss of generality.

Finally, we can prove Theorem \ref{thm:finite-sample}.
\begin{theorem*}[Theorem \ref{thm:finite-sample} (formal)]
Let $\epsilon_0 = \min_{U : S\not \subset U} \err(U) - \err(S)$.
Suppose we have an unlabeled dataset of $m$ i.i.d. pretraining points $\{(Z_i, Y_i)\}_{i=1}^m$, with:
\[
m > \frac{2\left(\log {d \choose |S|} + \log \frac{4}{\delta}\right)}{\epsilon_0^2},
\]
and a labeled dataset of $n$ downstream training points $\{(X^t_i, Y^t_i)\}_{i=1}^n$.
Let $\cG$ be all sets of size $|S|$ features chosen from the full set of $d$ features.
Let $\hat{g}$ be the minimizer of the empirical OCP pretraining objective:
\[
\hat{g} = \argmin_{g \in \cG} \hat{R}_{ord}(g)
\]
Let $\hat{f}$ be the minimizer of the empirical downstream objective when using the fixed representation $\hat{g}$:
\[
\hat{f} = \argmin_{f \in \cF} \hat{R}(f, \hat{g})
\]
Then for any $\delta > 0$, with probability at least $1-\delta$, $(\hat{f}, \hat{g})$ has excess risk:
\[
R(\hat{f}, \hat{g}) - \inf_{(f,g) \in \cF \times \cG} R(f,g) \le O\left(\sqrt{\frac{\VC(\cF) + \log(\frac{2}{\delta})}{n}}\right)
\]
\end{theorem*}
\begin{proof}
First, observe that Lemma \ref{lemma:S-unique} implies that $\epsilon_0 > 0$.
We identify the functions $g \in \cG$ with the sets that they select.
\citet[Theorem 2.13, p.20,][]{mohri2018foundations} implies that the amount of unlabeled data is sufficient so that
$R_{ord}(\hat{g}) < R_{ord}(S) + \epsilon_0$ with probability at least $1-\delta/2$.
By Lemma \ref{lemma:S-unique}, this implies $\hat{g} = S$.
Then we can learn downstream without incurring any cost compared to an optimal pair $(f^*, g^*)$, since $\hat{g} = g^* = S$.
Hence, the standard uniform convergence argument for $\VC$-classes (e.g., \citet[][Theorem 6.8]{shalev2014understanding}) implies the excess risk bound for $(\hat{f}, \hat{g})$ holds with probability at least $1-\delta/2$.
Combining the probabilities of failure for each step gives total failure probability at most $\delta$.

\end{proof}

\subsection{PCL versus OCP}
Consider a feature $X^t_i \in \{0,1\}$ such that $\P[X_i^1 = 1] = 1/2$, and $X^{t+1}_i = 1-X^t_i$.
This is a periodic feature with period 1.
$X_i^t$ qualifies as a background feature, since it is independent of all other features $X_j$, and $P[X_i^t = v, X_i^{t+1} = v'] = P[X_i^t = v', X_i^{t+1} = v]$ for all $(v,v') \in \{0,1\}^2$.
By the results from the previous section, inclusion of this feature therefore does not affect whether OCP finds the correct representation $S$.
In OCP with consecutive windows, $X^W_i \ne X^{W'}_i$ with probability 1 regardless of $Y$, so that feature is not useful for distinguishing between correctly and incorrectly-ordered pairs.
However, this feature is very helpful for PCL.
Because the positives for PCL are always consecutive, but the negatives are pairs of random elements, $\P[Y=-1 | X^W_i = X^{W'}_i] = 1$.
This biases PCL towards inclusion of feature $i$ over inclusion of more rarely-occurring features in $S$, so there are simple examples where PCL provably never recovers $S$ even in the infinite-sample limit.
Additionally, even when PCL works well in the infinite data limit, features that are \emph{weakly} predictive of whether elements are consecutive or not can affect the unlabeled sample complexity of PCL, as we show empirically in Appendix \ref{apdx:synth-experiments}.

\subsection{Biased OCP}
The OCP algorithm samples windows $(W,W+1)$ when $Y=1$ and $(W+1, W)$ when $Y=-1$.
Consider instead a version of OCP with ``bias'' in the negative distribution.
When $Y=1$, we still sample $(W,W+1)$ as the window pair.
But when $Y=-1$, we sample a \emph{random} pair $(W,W')$ with $|W-W'|=1$.
As with PCL, half of the negatives for this ``OCP-biased'' sampling are actually in the correct order.
Does this affect the representation learned during pretraining, and does it affect the unlabeled sample complexity required for learning?

Now we give two lemmas analogous to lemmas \ref{lem:S-optimal} and \ref{lemma:S-unique}.
\begin{lemma}
\label{lem:biased-S-optimal}
Let \[
h^*(v,v') = \argmax_{y\in\{-1,1\}} \P[Y=y | X^W = v, X^{W'} = v']
\]
be the Bayes-optimal classifier for the OCP-biased objective.
There exists a classifier $h(v_S,v'_S)$, depending only on the $S$ coordinates of $v$ and $v'$, achieving the same OCP-biased error as $h^*$.
\end{lemma}

\begin{lemma}
\label{lemma:biased-S-unique}
For $U \subset [d]$, define the error of $U$ as the best OCP-biased loss achievable when using the features in the set $U$:
\[
\err(U) = \inf_{c \text{ measurable}} \P[Y \ne c(X^W_U, X^{W'}_U)].
\]
For any $p, q \in \{0,1\}^{|U|}$ with $p_{U\cap S} = q_{U\cap S}$, define:
\[
m'_U(p,q) = \min\left(\frac{1}{3}\P[X^W_S \subsetneq X^{W'}_S | X^W_U = p, X^{W'}_U=q]; \P[X^{W'}_S \subsetneq X^{W}_S | X^W_U = p, X^{W'}_U=q]\right).
\]
The expected value of $m'_U(p,q)$ is given by:
\[
\E_{p,q}[m'_U(p,q)] = \sum_{p,q \in \{0,1\}^{|U|}}\mathbb{I}[p_{U\cap S} = q_{U\cap S}]\P[X^W_U=p, X^{W'}_U = q]m'_U(p,q)
\]
Then for any $U$,
\begin{align*}
\err(U) &= \frac{1}{2}\P[X^W_S = X^{W'}_S] + \frac{1}{3}\P[X^W_S \subsetneq X^{W'}_S] + \E_{p,q}[m'_U(p,q)]\\
&= \err(S) + \E_{p,q}[m'_U(p,q)]
\end{align*}
\end{lemma}
The proofs of these lemmas are entirely analogous to the proofs of lemmas \ref{lem:S-optimal}, \ref{lemma:S-unique} with some extra handling of the case of the $X^W_S \subsetneq X^{W'}_S$, since in OCP-biased those samples could also be negatives.

Now we can prove Theorem 2:
\begin{theorem*}[Theorem 2 (formal)]
Suppose Assumptions \ref{asmp:s-stays-active}-\ref{asmp:U-asmp} are satisfied.
Then $S$ is the unique optimal representation for the OCP-biased objective of size $|S|$.
\end{theorem*}
\begin{proof}
By direct comparison of the definition of $m_U(p,q)$ in Lemma \ref{lemma:S-unique} and $m'_U(p,q)$ in \ref{lemma:biased-S-unique}, we see that $\E_{p,q}[m_U(p,q)] > 0 \implies \E_{p,q}[m'_U(p,q)] > 0$.
Therefore, Lemma \ref{lem:m-positive} shows that Assumption \ref{asmp:U-asmp} implies $\E_{p,q}[m'_U(p,q)] > 0$ for all $U$ with $S \not \subset U$.
Hence, $S$ is the unique optimal representation for OCP-biased.
\end{proof}

Since $S$ is the unique optimal representation for OCP-biased, an entirely similar analysis to the proof of Theorem \ref{thm:finite-sample} yields a finite-sample bound for OCP-biased.
However, recall that the bound on the \emph{unlabeled} sample complexity depends on $1/\epsilon_0^2$, where $\epsilon_0 = \min_{U : S\not\subset U} \err(U) - \err(S)$.
By comparing the error formulae in Lemmas \ref{lemma:S-unique} and \ref{lemma:biased-S-unique} (in particular, the difference between $m_U$ and $m'_U$), we see that for any fixed $U$ with $S \not\subset U$, $\err_{biased}(U) - \err_{biased}(S) < \err_{ocp}(U) - \err_{ocp}(S)$.
Therefore, for the same distribution over $X$, the $\epsilon_0$ for OCP is larger than the $\epsilon_0$ for OCP-biased. This results in a better upper bound on the unlabeled sample complexity required to find a good representation.
In the following section, we show empirically that OCP-biased indeed requires more samples to find a good representation, but ultimately finds the same representation as OCP.

\clearpage
\section{Synthetic Experiments}
\label{apdx:synth-experiments}

\subsection{Synthetic Distributions}
In this section we describe the distributions for the synthetic experiments in Section \ref{sec:synthetic}. 
Both of these distributions are members of the class of distributions from Section \ref{sec:model}.
We proved in Appendix \ref{apdx:theory} that all of these distributions satisfy Assumptions \ref{asmp:s-stays-active}-\ref{asmp:U-asmp}, and that $S$ is the optimal representation for OCP on these distributions.

\subsubsection*{Distribution 1 (Figure 2a)}
Our first synthetic distribution includes trajectories $X$ of length 10 with $d=8$ features that are generated as described below: 
\begin{itemize}{\itemsep=0em}
    \item The set $S$ consists of 4 time-irreversible features. Each feature in $S$ has a fixed probability of activating (switching from `0' to `1') over the entire trajectory, independent of the other features in $S$. Activation time was chosen uniformly over the whole trajectory and independently per feature, and once activated, features remained on. In our synthetic data, probabilities were 0.4 for the first two features and 0.6 for the next two features. 
    \item The set $\hat{S}$ consisted of noisy versions of the first three variables in $S$. $\epsilon_i$ was set to 0.7 for all variables.
    \item The last feature $X_{-1}$ was a background, reversible feature in $B$. It was a periodic function, alternating between 0 and 1, with uniform initialization over $\{0,1\}$.
\end{itemize}
\subsubsection*{Distribution 2 (Figure 2b)}
Similarly, our second synthetic distribution includes trajectories $X$ of length 10 with $d=7$ features that are generated as described below: 
\begin{itemize}{\itemsep=0em}
    \item The set $S$ consists of 4 time-irreversible features, identical to before. Each feature in $S$ has a fixed probability of activating (switching from `0' to '1') over the entire trajectory, independent of the other features in $S$. Activation time was chosen uniformly over the whole trajectory and independently per feature, and once activated, features remained on. In our synthetic data, probabilities were 0.4 for the first two features and 0.6 for the next two features. 
    \item The set $\hat{S}$ consisted of noisy versions of the first two variables in $S$. $\epsilon_i$ was set to 0.55 for all variables.
    \item The last feature $X_{-1}$ was a background, reversible feature in $B$. $X_{-1}^0$ was sampled uniformly from $\{0,1\}$ and the rest of $X_{-1}$ was set such that for all $t \geq 1$, $\P[X_{-1}^{t}=X_{-1}^{t-1}]=0.3$. 
\end{itemize} 
\subsection{Experimental Setup}
\subsubsection*{Synthetic Data Creation}
In order to quantify how many unlabeled pre-training samples are required to recover $S$, we created data sets $\{X_i\}_{i=1}^m$ with varying $m$, taking on values 50, 100, 200, 400, 600, 800, 1000, 2000, 4000, 8000, and 16,000. For each dataset size $m$, we generated 100 independently drawn sets, according to the distributions described previously.

For each created set, a single pair was sampled from each trajectory according to OCP, PCL, or OCP-biased sampling. 
I.e. an unlabeled dataset of $m$ trajectories produced a pre-training dataset with $m$ pairs.

\subsubsection*{Selection of Optimal Representation}
Now, for each of our pre-training datasets and each sampling scheme (OCP, PCL, OCP-biased), we now determine what feature representation $\hat{g}$ would be selected via pre-training for each one. 

In order to do so, we iterate over each possible $g\in \cG$ where $\cG = \{U \subset [d] : |U| = |S| = d_0\}$. After subselecting to the features in $g$, each pair of data points $(X^t, X^{t'})$ was featurized as input as $[X^t; X^{t'}; X^t-X^{t'}; |X^t-X^{t'}|]$. The loss is then minimized via the LogisticRegression implementation from scikit-learn with the `liblinear' optimizer. We then select the $\hat{g}$ that minimizes the empirical 
pre-training loss and calculate the overlap with the true features in $S$. 

\subsection{Explanation of Observed Behavior}
\subsubsection*{Distribution 1}
In Distribution 1, we see that OCP essentially always converges to the optimal representation with 8,000 data points, and OCP-biased also coverges, albeit slower. 
However, PCL is never able to break the barrier of 3 features, since it opts to choose the periodic feature $X_{-1}$ instead. 
This feature (which alternates between 0 and 1) is highly discriminative for the PCL pre-training task, since $|X_{-1}^t-X_{-1}^{t+1}|=1$ whenever $Y=1$, and the same is true only half the time when $Y=0$. 
Therefore, while it is a background, reversible feature that may not be useful for a downstream task, it is useful for the PCL task, and therefore PCL fails to find the optimal representation of time-irreversible features. 
In contrast, for OCP, $|X_{-1}^t-X_{-1}^{t+1}|=1$ is true across all examples, and therefore would not be chosen as a discriminative feature. 
\subsubsection*{Distribution 2}
In Distribution 2, we again see that all methods are able to identify the optimal representation. 
However, OCP requires fewer pre-training samples in order to reach the optimal representation.
While not periodic as in Distribution 1, $X_{-1}$ is again weakly predictive of whether or not a window is consecutive.
While $X_{-1}$ is not as strongly predictive as in Distribution 1, PCL still opts to select it in $\hat{g}$ in the lower-data regime. 
However, with sufficient data, PCL overcomes the ``false'' signal to opt for the ``correct'' feature instead.  

\clearpage
\section{Real-world Experiments}
\label{apdx:experiments}

\subsection{Pre-training for feature selection}
\subsubsection*{Implementation Details}
For the linear pre-training for feature selection experiment, we utilized the LogisticRegression implementation from scikit-learn with the `liblinear' optimizer and balanced class reweighting \citep{scikit-learn}. Hyperparameter tuning was conducted on the validation set independently for each feature subset, dataset size, and fold number. Hyperparameters included regularization scheme (`l1' or `l2') and regularization constant ($10^{-3}$ to $10^6$). All performance reported is on the held-out test sets using the best hyperparameter setting from the validation set. 

\subsubsection*{Granular Experimental Results}
Below, we present a more granular view on results including standard deviations, focusing on the difference between direct downstream prediction and using order contrastive pre-training to select features.  

\begin{table}[h]
    \centering
        \caption{AUC $\pm$ standard deviations (as calculated over five folds), comparing the use of all features to just those selected by OCP. The final row displays the difference in AUC between OCP-selected features and all features and the standard deviation of that difference.}
\begin{tabular}{l|lllll|}
\cline{2-6}  & \multicolumn{5}{c|}{\textit{Fraction of training data}}     \\ \hline
\multicolumn{1}{|l|}{\textit{Features}} & \multicolumn{1}{l}{\textbf{1}} & \multicolumn{1}{l}{\textbf{1/2}} & \multicolumn{1}{l}{\textbf{1/4}} & \multicolumn{1}{l}{\textbf{1/8}} & \multicolumn{1}{l|}{\textbf{1/16}} \\ \hline
    \multicolumn{1}{|l|}{\textbf{OCP subset}}
    & 0.864   $\pm$ .022  
    & 0.860     $\pm$ .026
    & 0.847   $\pm$ .024
    & 0.808     $\pm$ .023 
    & 0.786 $\pm$ .068 \\ \hline 
    \multicolumn{1}{|l|}{\textbf{All features}}   
    & 0.856  $\pm$ .021                        
    & 0.851  $\pm$ .021                          
    & 0.818      $\pm$ .036                      
    & 0.723  $\pm$ .057                           
    & 0.726  $\pm$ .074  \\\hline 
\multicolumn{1}{|l|}{\textbf{OCP - All}}  
    & 0.008   $\pm$ .014                        
    & 0.008  $\pm$ .039                           
    & 0.029  $\pm$ .035                  
    & 0.082 $\pm$ .058                            
    & 0.054  $\pm$ .050            \\
\hline
\end{tabular}

    \label{tab:my_label}
\end{table}

\subsubsection*{Average Precision Results}
In addition to the note-level AUC, we also provide results via a different patient-level precision metric; this provides another interpretable view on performance, since each patient has multiple notes in the test set, and the labels are imbalanced. We define \textit{Average Precision at 80\% recall} in the following manner. We first find the threshold at which 80\% of positive labels (displays progression) would be recovered. Then per patient, we calculate the precision of the retrieved notes, assuming use of that threshold, which we then average over patients. If no notes are surfaced for a patient, we set the precision to 1 if no positive note exists, and 0 otherwise.

We then assess performance using this precision-level metric. We follow the same procedure as before, except that hyperparameter settings are now chosen on the basis of this precision metric on the validation set, instead of AUC. In the table below, we note the same trends are present with average precision as with AUC; namely, there is essentially no difference when all training data can be used, but a much larger difference when the model is restricted to only a fraction.

\begin{table}[h]
    \centering
        \caption{Average precision at 80\% recall $\pm$ standard deviations (as calculated over five folds), comparing the use of all features to just those selected by OCP. The final row displays the difference in average precision between OCP-selected features and all features.}
\begin{tabular}{l|lllll|}
\cline{2-6}  & \multicolumn{5}{c|}{\textit{Fraction of training data}}     \\ \hline
\multicolumn{1}{|l|}{\textit{Available features}} & \multicolumn{1}{l}{\textbf{1}} & \multicolumn{1}{l}{\textbf{1/2}} & \multicolumn{1}{l}{\textbf{1/4}} & \multicolumn{1}{l}{\textbf{1/8}} & \multicolumn{1}{l|}{\textbf{1/16}} \\ \hline
    \multicolumn{1}{|l|}{\textbf{OCP subset}}     & 0.54   $\pm$ 0.09                        & 0.54     $\pm$ 0.07                        & 0.51  $\pm$ 0.08                          & 0.48     $\pm$ 0.09                        & 0.42 $\pm$ 0.10                            \\ \hline 
    \multicolumn{1}{|l|}{\textbf{All features}}   & 0.54  $\pm$ 0.12                         & 0.50  $\pm$ 0.12                           & 0.49      $\pm$ 0.05                       & 0.37  $\pm$ 0.07                           & 0.37 $\pm$ 0.07                           \\\hline 
\multicolumn{1}{|l|}{\textbf{OCP - All}}  & 0   $\pm$ 0.08                        & 0.05  $\pm$ 0.10                           & 0.02           $\pm$ 0.08                  & 0.10 $\pm$ 0.07                            & 0.05  $\pm$ 0.05                         \\
\hline
\end{tabular}
    \label{tab:ap_features}
\end{table}

\newpage 
\subsection{Nonlinear representations}
\subsubsection*{Implementation Details}
The masked language modeling was conducted using the BertForMaskedLM implementation from \citet{wolf-etal-2020-transformers} with a 15\% masking rate and a learning rate of 5e-5. The checkpoint used downstream was selected as the one in which the model had the lowest validation loss on a held-out set of notes. The contrastive pre-training across all 3 objectives was conducted using the BertForNextSentencePrediction implementation from \citet{wolf-etal-2020-transformers} with a learning rate of 1e-5 and weight decay of 0.01. For each contrastive approach, the model checkpoint with the highest validation accuracy on a held-out set of 3200 contrastive pairs was chosen for use downstream. The  L2-regularized linear layer was implemented using scikit-learn, using class-balanced reweighting and the `liblinear' optimizer. As before, the regularization (1e-2 to 1e5) was chosen for each seed, fold, model, and dataset size using the best performance on the validation set. 
\subsubsection*{Computational Burden}
Each pre-training method was trained on a single NVIDIA GeForce GTX 1080 Ti GPU with 12GB of memory. Convergence required about 6 hours for the masked language modeling pre-training and 3-4 hours for each of the contrastive pre-training methods. Since pre-training was conducted over 3 different seeds, it collectively required about 54 GPU hours. The forward pass to extract frozen embeddings from the different models across seeds and folds was minimal in time (less than 15 minutes). All downstream experiments and hyperparameter tuning involved solely a linear layer and were conducted on a CPU.

\subsubsection*{Granular Experimental Results}
Below, we display more granular experimental results, splitting performance by seed. We display both the mean and standard deviation AUC across the 5 folds at each training data size. While there is some variation between seeds, we find that OCP consistently outperforms at smaller dataset sizes. 

\begin{table}[h]
    \centering
        \caption{AUC average and standard deviation (as calculated over five folds) over each of the three seeds trained per method.}
\begin{tabular}{l|ccccc|}
\cline{2-6}  & \multicolumn{5}{c|}{\textit{Fraction of training data}}     \\ \hline
\multicolumn{1}{|l|}{\textit{AUC $\pm$ std dev}} & \multicolumn{1}{c}{\textbf{1}} & \multicolumn{1}{c}{\textbf{1/2}} & \multicolumn{1}{c}{\textbf{1/4}} & \multicolumn{1}{c}{\textbf{1/8}} & \multicolumn{1}{c|}{\textbf{1/16}} \\ \hline  \cline{1-6}
\multicolumn{1}{|l|}{\textbf{OCP Seed A}}   
& 0.87 $\pm$ .04    & 0.86 $\pm$ .04  &  0.84 $\pm$ .03 &  0.82 $\pm$ .03  & 0.83 $\pm$ .04        \\\hline  
\multicolumn{1}{|l|}{\textbf{OCP Seed B}}   
& 0.86 $\pm$ .03    & 0.85 $\pm$ .04  &  0.84 $\pm$ .04 &  0.81 $\pm$ .02  & 0.82 $\pm$ .03        \\\hline 
\multicolumn{1}{|l|}{\textbf{OCP Seed C}}   
& 0.88 $\pm$ .02    & 0.86 $\pm$ .04  &  0.85 $\pm$ .03 &  0.81 $\pm$ .03  & 0.79 $\pm$ .02        \\\hline  \cline{1-6}

\multicolumn{1}{|l|}{\textbf{BERT}}   
& 0.79 $\pm$ .06    & 0.74 $\pm$ .06  & 0.72 $\pm$ .04 & 0.63 $\pm$ .05 & 0.60 $\pm$ .06      \\\hline  \cline{1-6}

\multicolumn{1}{|l|}{\textbf{Fine-Tuned LM Seed A}}   
& 0.87 $\pm$ .05    & 0.85 $\pm$ .06  &  0.85 $\pm$ .02 &  0.77 $\pm$ .06  & 0.75 $\pm$ .06        \\\hline 

\multicolumn{1}{|l|}{\textbf{Fine-Tuned LM Seed B}}   
& 0.84 $\pm$ .03    & 0.82 $\pm$ .04  &  0.79 $\pm$ .02 &  0.70 $\pm$ .07  & 0.67 $\pm$ .09        \\\hline 

\multicolumn{1}{|l|}{\textbf{Fine-Tuned LM Seed C}}   
& 0.83 $\pm$ .02    & 0.80 $\pm$ .01  &  0.77 $\pm$ .01 &  0.73 $\pm$ .06  & 0.71 $\pm$ .05        \\\hline  \cline{1-6}

\multicolumn{1}{|l|}{\textbf{Pt-Contrastive Seed A}}   
& 0.82 $\pm$ .04    & 0.79 $\pm$ .03  &  0.77 $\pm$ .04 &  0.69 1 $\pm$ .07  & 0.66 $\pm$ .10        \\\hline   

\multicolumn{1}{|l|}{\textbf{Pt-Contrastive Seed B}}   
& 0.85 $\pm$ .02    & 0.83 $\pm$ .04  &  0.78 $\pm$ .04 &  0.72 $\pm$ .03  & 0.69 $\pm$ .1        \\\hline   

\multicolumn{1}{|l|}{\textbf{Pt-Contrastive Seed C}}   
& 0.84 $\pm$ .05    & 0.85 $\pm$ .05 &  0.83 $\pm$ .03 &  0.77 $\pm$ .04  & 0.72 $\pm$ .06       \\\hline  \cline{1-6} 

\multicolumn{1}{|l|}{\textbf{PCL Seed A}}   
& 0.88 $\pm$ .02    & 0.88 $\pm$ .03 &  0.85 $\pm$ .02 &  0.83 $\pm$ .04  & 0.78 $\pm$ .03           \\\hline 

\multicolumn{1}{|l|}{\textbf{PCL Seed B}}   
& 0.86 $\pm$ .01    & 0.85 $\pm$ .04  &  0.80 $\pm$ .06 &  0.77 $\pm$ .02  & 0.76 $\pm$ .04        \\\hline 

\multicolumn{1}{|l|}{\textbf{PCL Seed C}}   
& 0.86 $\pm$ .03    & 0.84 $\pm$ .03  &  0.80 $\pm$ .04 &  0.75 $\pm$ .04  & 0.71 $\pm$ .08         \\\hline  
\end{tabular}
    \label{tab:seed_results}
\end{table}

\newpage 

\subsubsection*{Average Precision Results}
Below we show results on the test set using the patient-level average precision metric introduced in Appendix B.1. As before, hyperparameter settings are chosen based on the model with best performance on the validation set, per this precision metric. Again, we find that OCP is relatively consistent in performance even with minimal training data, and outperforms other methods, particularly in the low data regime.
\begin{table}[h]
\setlength{\tabcolsep}{4.5pt}
    \centering
        \caption{Performance in terms of average precision at 80\% recall as calculated across 5 folds and 3 seeds for fine-tuned methods. The first row contains the mean average precision of OCP $\pm$ its standard deviation. The following rows contain the mean precision advantage of OCP over each comparison method, and the percentage of time OCP outperforms that method, computed over all seeds and folds.}
\begin{tabular}{l|ccccc|}
\cline{2-6}  & \multicolumn{5}{c|}{\textit{Fraction of training data}}     \\ \hline
\multicolumn{1}{|l|}{\textit{Prec diff (OCP Win \%)}} & \multicolumn{1}{c}{\textbf{1}} & \multicolumn{1}{c}{\textbf{1/2}} & \multicolumn{1}{c}{\textbf{1/4}} & \multicolumn{1}{c}{\textbf{1/8}} & \multicolumn{1}{c|}{\textbf{1/16}} \\ \hline
\multicolumn{1}{|l|}{\textbf{OCP Prec}}   
& 0.57 $\pm$ .08    & 0.58 $\pm$ .07  &  0.55 $\pm$ .06 &  0.50 $\pm$ .08  & 0.50 $\pm$ .09        \\\hline  \cline{1-6}
\multicolumn{1}{|l|}{\textbf{OCP - BERT}}   
& 0.16 (100\%)    & 0.19 (100\%)  & 0.19 (100\%) & 0.19 (100\%) & 0.22 (100\%)       \\\hline  
\multicolumn{1}{|l|}{\textbf{OCP - FT LM}}   
& 0.09 (80\%)  & 0.09 (80\%) & 0.09 (84\%) &  0.11 (87\%) & 0.13 (93\%)     \\\hline  
\multicolumn{1}{|l|}{\textbf{OCP - Pt-Contrastive}}   
& 0.08 (73\%)   & 0.09 (80\%) & 0.14 (100\%) & 0.15 (96\%) & 0.17 (98\%)  \\\hline  

\multicolumn{1}{|l|}{\textbf{OCP - PCL}}   
& 0.02 (56\%)   & 0.04 (69\%) & 0.07 (73\%)  & 0.10 (84\%)   & 0.13 (91\%)   \\\hline  
\end{tabular}
    \label{tab:ap_deep}
\end{table}

\subsubsection*{Qualitative Analysis}
We now conduct a qualitative analysis to understand whether the OCP pre-trained model is operating as expected, namely whether it is attending to those features we would expect to be most crucial both in ordering and for downstream progression extraction.

The BertForNextSentencePrediction model used for contrastive pre-training is implemented such that sequence classification is based off of the representation of the \texttt{CLS} classifier token. Therefore, we examine those tokens which are most highly attended to by the \texttt{CLS} token in the last BERT layer, as a proxy signal for what the model is attending to for its final classification representation. Over all examples in the validation set, we find the average attention each token contributed in the final layer. There are then 21 tokens with an average attention above a threshold of 0.1; they include \textit{increased/increase/increasing}, \textit{change}, \textit{unchanged}, \textit{no}, \textit{stable}, \textit{negative}, \textit{persistent}, and \textit{resolved}. All such features elucidate disease stage, and therefore, it seems qualitatively plausible that OCP learns useful downstream representations in the nonlinear case. Other highly ranked features include the \texttt{SEP} token used to split the contrastive pairs, as well as \textit{prior}, which is a possible leaky feature that indicates order information, but is not as useful for downstream analysis.

\end{document}